\documentclass{article} % For LaTeX2e
\usepackage{iclr2027_conference,times}

\usepackage{amsmath,amsfonts,bm}

\def\eqref#1{equation~\ref{#1}}
\def\1{\bm{1}}

\DeclareMathAlphabet{\mathsfit}{\encodingdefault}{\sfdefault}{m}{sl}
\SetMathAlphabet{\mathsfit}{bold}{\encodingdefault}{\sfdefault}{bx}{n}

\usepackage{hyperref}        % hyperlinks
\usepackage{url}             % simple URL typesetting
\usepackage{graphicx} % Required for inserting images
\usepackage{natbib}
\usepackage{capt-of}
\usepackage{bbm}

\usepackage[utf8]{inputenc} % allow utf-8 input
\usepackage[T1]{fontenc}    % use 8-bit T1 fonts
\usepackage{booktabs}       % professional-quality tables
\usepackage{amsfonts}       % blackboard math symbols
\usepackage{nicefrac}       % compact symbols for 1/2, etc.
\usepackage{microtype}      % microtypography
\usepackage{xcolor}         % colors
\usepackage{algorithm}
\usepackage{algpseudocode}

\usepackage{booktabs}
\usepackage{multirow}
\hypersetup{hidelinks}

\usepackage{amsfonts}
\usepackage{amsmath}
\usepackage{amsthm}
\newtheorem{theorem}{Theorem}
\newtheorem{lemma}{Lemma}
\newtheorem{corollary}{Corollary}

\title{Identifiable World Models from Pretrained Diffusion Representations
}
\author{
	Ruchi Sandilya$^{*}$, Conor Liston \& Logan Grosenick \\
	Department of Psychiatry and Brain \& Mind Research Institute \\
	Weill Cornell Medicine, New York, NY, USA \\
	$^{*}$Corresponding author: \texttt{sar4018@med.cornell.edu}
}

\iclrfinalcopy % Uncomment for camera-ready version, but NOT for submission.

\begin{document}

\maketitle
\begin{abstract}
Diffusion-based world models can generate and predict trajectories in high-dimensional dynamical systems, but predictive accuracy does not imply that their latent coordinates recover the underlying state variables or causal interactions. We ask whether a frozen pretrained diffusion model can be equipped with identifiable coordinates without retraining its generative backbone. We show that auxiliary-variable nonlinear ICA guarantees can be transferred to Contrastive Diffusion Alignment (ConDA), which learns only a lightweight alignment map on top of frozen diffusion latents. Under standard TCL/GCL assumptions, the aligned representation identifies latent dynamical states up to permutation and componentwise invertible transformations, preserves the latent dynamic structural causal model, and reduces lagged graph recovery to transition-Jacobian sparsity. We evaluate TCL-, GCL-, and CEBRA-based ConDA against TDRL, CaRiNG, IDOL, temporal SuaVE, and iVAE across physical and robotic video systems. TCL and GCL achieve near-perfect blockwise state recovery and competitive lagged graph recovery, including exact recovery in a simulated falling-body system. In a simulated bipedal robot, learned dynamics recover the sign and temporal structure of responses to held-out control perturbations. These results show that a frozen generative diffusion model can be equipped with coordinates that are identifiable, structurally interpretable, and useful for analyzing intervention-relevant dynamics.

\end{abstract}

% For Too Long; Didn't Read": a short sentence describing your paper
%We show that diffusion-based world models can recover identifiable latent causal dynamics, including causal structure and interventional responses, rather than merely generating accurate trajectories.

\section{Introduction}
Learning meaningful representations of high-dimensional dynamical systems is a central problem in scientific machine learning. Many scientific, biomedical, and robotic systems are governed by low-dimensional latent dynamics but are observed through high-dimensional measurements. Diffusion models provide powerful generative representations of such observations \citep{ho2020denoising,rombach2022high}, but generative fidelity alone does not imply that their internal representations recover the variables or interactions governing the underlying system. Indeed, multiple latent parameterizations can explain the same observations \citep{hyvarinen1999nonlinear,locatello2019challenging,khemakhem2020variational}, and predictive representations need not preserve the causal structure required for intervention, control, or scientific reasoning \citep{bengio2013representation,scholkopf2021toward}. Contrastive learning, in turn, can exploit temporal or auxiliary information to recover structured low-dimensional representations \citep{hyvarinen2016unsupervised,oord2018representation,chen2020simple,schneider2023cebra}. This raises a natural question: can a pretrained diffusion model retain its generative capacity while supporting an identifiable and controllable latent coordinate system that preserves causal structure?

Nonlinear ICA provides one route to identifiability through auxiliary variables
such as time, conditions, interventions, or control inputs
\citep{hyvarinen2017nonlinear, hyvarinen2019nonlinear, khemakhem2020variational},
but its extension to high-dimensional dynamical systems represented through
diffusion latents remains largely unexplored. Building on Contrastive Diffusion Alignment (ConDA) \citep{sandilya2026contrastivediffusionalignmentlearning}, we
study identifiability and causal graph recovery in diffusion-based world models. We distinguish three representation levels: the true low-dimensional dynamical state $s_t$, the high-dimensional diffusion latent $x_t$, and the observed frame $y_t$. Generatively, the system follows ($s_t \to x_t \to y_t$). For identification, we invert the observation into diffusion space and
learn an auxiliary-aligned embedding $z_t = h(x_t)$ ($y_t \to x_t \to z_t$). Identifiability then asks whether the learned coordinates $z_t$ recover the underlying state $s_t$ up to the admissible ambiguities of permutation and componentwise invertible transformations.

We use the standard density-ratio interpretation of supervised InfoNCE to connect auxiliary-aligned contrastive learning on diffusion latents with auxiliary-variable nonlinear ICA. Under the stated separability and identifiability assumptions, the learned embedding recovers the latent dynamical state up to permutation and componentwise invertible transformations. We further show that these coordinates preserve the factorization and parent structure of the latent dynamic SCM, and that the lagged causal graph is characterized by the sparsity of the inverse transition-noise Jacobian. 
In practice, under additive-noise dynamics, we estimate the corresponding graph using the sparsity of a learned forward-dynamics Jacobian.
Empirically, we evaluate TCL-, GCL-, and CEBRA-based ConDA embeddings \citep{hyvarinen2016unsupervised, hyvarinen2019nonlinear, schneider2023cebra}
on CausalVerse physical and robotic simulations with known ground-truth dynamics \citep{causalverse2025}, as well as on Microduck, a controlled bipedal locomotion environment based on the official Pollen Robotics reinforcement-learning framework \citep{microduck_rl}.  We assess whether these representations recover low-dimensional latent state structure and support lagged causal graph recovery; in Microduck, we additionally evaluate interventional prediction under held-out control perturbations. Overall, our framework evaluates world models beyond reconstruction and forecasting through latent identifiability, causal graph recovery, and preservation of intervention-relevant dynamics.

The role of diffusion in our framework is not to claim that diffusion latents are universally superior representations for state recovery. Rather, diffusion supplies a pretrained generative world model capable of synthesizing trajectories. Our question is whether this generative representation can be equipped, without retraining the backbone, with an identifiable coordinate system suitable for causal analysis and intervention. Accordingly, \(z_t\) is deliberately not a replacement generative representation: it is an identifiable coordinate chart on the frozen generative representation \(x_t\).

\begin{figure}[t]
    \begin{center}
    \includegraphics[width=\linewidth]{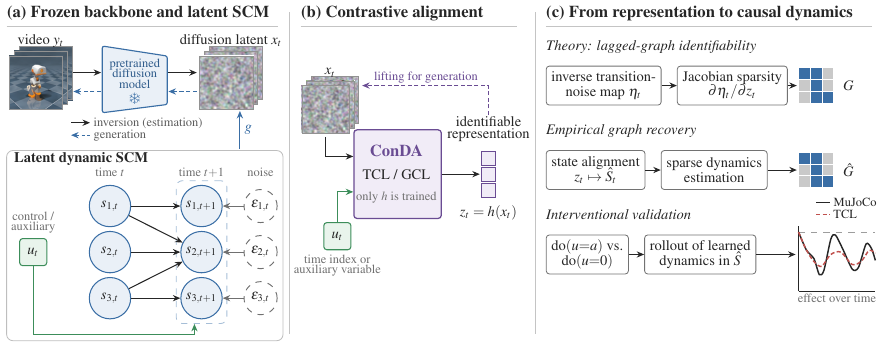}
    \end{center}
    %\caption{\textbf{Overview of the proposed framework.} A frozen pretrained diffusion model maps observed video frames \(y_t\) to diffusion latents \(x_t\). Temporal and auxiliary information \(u_t\) are then used by ConDA, instantiated with TCL, GCL, or CEBRA, to learn an identifiable representation \(z_t=h(x_t)\), recoverable up to permutation and componentwise invertible transformations while preserving the latent dynamic SCM. In the identifiable space, the lagged causal graph is theoretically identified by the sparsity pattern of the Jacobian of the inverse transition-noise map. Empirically, structural recovery is evaluated after Ridge alignment of the learned representation to ground-truth state coordinates, followed by residual Lasso and trajectory-bootstrap stability selection. Interventional validity is assessed separately by comparing learned and simulator responses under controlled interventions.}
  \caption{\textbf{Overview of the proposed framework.}
  (a)~A frozen pretrained diffusion model serves as a generative world model: it embeds video frames $y_t$ into diffusion latents $x_t$ by inversion (solid arrows) and decodes latents back into frames (dashed arrows). The latents are in turn generated by a lagged dynamic SCM over the latent state $s_t$, driven by inputs $u_t$ and independent noise, through $x_t = g(s_t)$.
  (b)~TCL- or GCL-based ConDA trains only a lightweight map $h$, using temporal or auxiliary information, to obtain $z_t = h(x_t)$. Under nonlinear ICA assumptions, $z_t$ identifies $s_t$ up to permutation and componentwise invertible transformations, and can be lifted back to $x_t$ for generation.
  (c)~Lagged causal structure is characterized theoretically by inverse transition-noise Jacobian sparsity and estimated empirically from states aligned from $z_t$. Dynamics learned in these coordinates are validated by predicting responses to held-out control perturbations.}
  \label{fig:method}
\end{figure}

\textbf{Contributions.}
(1) We formulate identifiability and causal graph recovery directly in the latent space of a frozen pretrained diffusion model, separating the underlying dynamical state $s_t$, diffusion representation $x_t$, observed frame $y_t$,
and auxiliary-aligned embedding $z_t=h(x_t)$. This decouples the generative world model from causal identification, as the diffusion backbone retains its capacity to generate trajectories, while the learned alignment supplies identifiable coordinates for analyzing and controlling its dynamics.
(2) We establish sufficient conditions under which TCL- or GCL-based alignment on frozen diffusion latents identifies the latent dynamical state up to permutation and componentwise invertible transformations, preserves the factorization and lagged parent structure of the latent dynamic SCM, and characterizes lagged causal edges through sparsity of the inverse transition-noise Jacobian.
(3) We evaluate identifiability, lagged causal graph recovery, and interventional prediction across physical and robotic video dynamical systems. TCL- and GCL-based ConDA variants achieve strong, and in several settings
near-perfect, blockwise state recovery and support recovery of lagged causal structure, including exact graph recovery in the simulated falling-body system. On controlled Microduck dynamics, we further evaluate whether the learned representations preserve responses to held-out control
perturbations. Together, these experiments provide an evaluation of diffusion-based world models beyond reconstruction and forecasting, using state recovery, structural recovery, and intervention-relevant dynamics.

\section{Related Work}
\textbf{World models and diffusion representations.}
World models learn latent representations that support prediction, planning, and
control in high-dimensional environments
\citep{ha2018world,hafner2019learning,hafner2020dream,hafner2023mastering}.
While effective for forecasting and decision making, these representations are
not generally identifiable and need not recover the underlying physical state or
causal dependencies required for mechanistic interpretation and intervention.
Diffusion models provide powerful generative representations for high-dimensional
data \citep{sohl2015deep,ho2020denoising,song2021score}, while latent diffusion
models improve scalability by operating in compressed representation spaces
\citep{rombach2022high}. However, diffusion representations are primarily optimized
for generative fidelity rather than recovery of identifiable dynamical states or
causal structure. We therefore study whether auxiliary-aligned embeddings of
frozen pretrained diffusion representations can recover low-dimensional latent
states and the lagged causal dependencies governing their dynamics.

\textbf{Contrastive learning and nonlinear ICA.}
Contrastive learning learns low-dimensional representations by bringing
positive pairs together while separating negative pairs
\citep{oord2018representation,chen2020simple,he2020momentum}. For dynamical
systems, TCL exploits temporal nonstationarity to recover latent sources under
nonlinear ICA assumptions
\citep{hyvarinen2016unsupervised,hyvarinen2017nonlinear}, while generalized
contrastive learning (GCL) uses auxiliary variables to establish nonlinear ICA
identifiability under broader forms of auxiliary information
\citep{hyvarinen2019nonlinear}. CEBRA similarly learns low-dimensional
embeddings aligned with behavioral, temporal, or neural variables
\citep{schneider2023cebra}. More generally, nonlinear latent variables are not
identifiable without additional assumptions
\citep{hyvarinen1999nonlinear,locatello2019challenging}; auxiliary information
such as time, labels, experimental conditions, or control inputs can provide
the structure required for identifiability
\citep{hyvarinen2019nonlinear,khemakhem2020variational}. We build on these
ideas by applying auxiliary-aligned contrastive objectives to frozen diffusion
latents and using the density-ratio interpretation of contrastive learning to
connect the learned embeddings to auxiliary-variable nonlinear ICA.

\textbf{Causal representation learning and causal discovery.}
Causal representation learning seeks latent representations that support causal
reasoning, intervention, and transfer
\citep{bengio2013representation,scholkopf2021toward}. Standard learned
representations may capture predictive or statistical structure without preserving
the underlying causal mechanisms. Causal discovery, in turn, aims to recover
directed dependencies from observational or interventional data
\citep{spirtes2000causation,pearl2009causality,peters2017elements}. In
time-series and dynamical systems, temporal ordering and lagged dependencies
provide additional structure for identifying directed interactions
\citep{granger1969investigating,runge2019detecting,tank2021neural}. However,
many causal discovery methods assume access to meaningful state variables,
whereas in high-dimensional systems the relevant dynamical variables may
themselves be latent. We address this gap by studying when embeddings learned
from diffusion representations recover identifiable latent states and preserve
their lagged causal structure. Under our dynamic SCM assumptions, lagged causal
edges are characterized by the sparsity pattern of the Jacobian of the inverse
transition-noise map. This characterization is related to Jacobian-based causal
discovery under nonlinear ICA \citep{reizinger2023jacobian}, adapted here to
auxiliary-aligned representations of frozen diffusion latents in dynamical
systems.

%\textbf{Positioning of this work.} 
% Our theoretical components individually adapt existing results:
% Theorem~\ref{thm:nonlinear-ica} specializes auxiliary-variable nonlinear ICA \citep{hyvarinen2019nonlinear, khemakhem2020variational} to first-order Markov dynamics, Lemma~\ref{lemma-density-ratio} restates the standard density-ratio interpretation of supervised InfoNCE \citep{oord2018representation}, and Theorem~\ref{thm:graph-recovery} adapts Jacobian-based causal discovery \citep{reizinger2023jacobian} to the dynamical setting. Prior causal representation learning approaches for time-series data \citep{yao2022temporally, chen2024caring, li2024identification} establish identifiability for encoders trained jointly with the dynamics objective. Our contribution is to extend these guarantees to \emph{dynamic diffusion models} using ConDA \citep{sandilya2026contrastivediffusionalignmentlearning}, where the encoder is a frozen pretrained diffusion latent rather than a jointly trained network. We show that auxiliary-aligned contrastive alignment of pretrained diffusion latents inherits identifiability up to permutation and componentwise invertible transformations (Theorem~\ref{thm:identifiability}), preserves the dynamic structural causal model (Theorem~\ref{thm:causal-rep}), and enables Jacobian-based recovery of lagged causal edges, evaluated on CausalVerse with both identifiability and graph-recovery criteria.
\textbf{Positioning of this work.} Our work addresses a setting largely unexplored in causal representation learning: recovering identifiable and causally meaningful states from \emph{frozen pretrained diffusion representations} of high-dimensional spatiotemporal systems. Rather than
learning the encoder jointly with the temporal or causal model \citep{yao2022temporally,chen2024caring,li2024identification}, we keep the
diffusion representation fixed and learn only a lightweight auxiliary-aligned mapping. Building on auxiliary-variable nonlinear ICA
\citep{hyvarinen2019nonlinear,khemakhem2020variational}, contrastive density-ratio estimation \citep{oord2018representation}, and Jacobian-based causal
discovery \citep{reizinger2023jacobian}, we show that the resulting embeddings identify latent dynamical states up to permutation and componentwise invertible transformations, preserve the factorization and lagged parent structure of the dynamic SCM, and characterize lagged causal edges through inverse transition-noise Jacobian sparsity. Empirically, TCL- and GCL-based ConDA variants achieve strong, often near-perfect, blockwise identifiability and recover lagged causal structure across CausalVerse systems, while Microduck experiments show that the learned representations also preserve the direction and temporal structure of responses to held-out control perturbations. Together, these results position pretrained diffusion representations as a reusable substrate for identifiable, structurally faithful, and intervention-relevant world models of high-dimensional dynamical systems.

\section{Method}
This section develops the theoretical connection between identifiable
representations learned from frozen diffusion latents, preservation of latent
dynamical structure and lagged causal graph recovery. Proofs of the main theoretical statements are
deferred to 
%This section presents the main theoretical statements; proofs are deferred to
Appendix \ref{app:nonlinear-ica-prrof}-\ref{app-jacobian-proof}.

%\subsection{Modeling Framework}
\textbf{Modeling Framework.}
We consider a latent dynamical system in which high-dimensional observations are generated from an underlying low-dimensional state:
\begin{align}
s_{t+1} = F(s_t,u_t,\varepsilon_t), \qquad
x_t = g(s_t), \qquad
y_t = r(x_t),
\end{align}
where $s_t \in \mathbb{R}^d$ denotes the true dynamical state,
$u_t$ denotes observed auxiliary information or externally imposed controls,
$x_t \in \mathbb{R}^p$ denotes the frozen diffusion representation, and
$y_t$ denotes the observed frame. Thus, the generative hierarchy is
$s_t \rightarrow x_t \rightarrow y_t.$
For identification and causal analysis, we learn a structured embedding
%\begin{equation}
$z_t = h(x_t)$,
%\end{equation}
yielding the estimation pipeline
$y_t \rightarrow x_t \rightarrow z_t.$
The embedding $z_t$ is not a generative layer; rather, it provides a
low-dimensional coordinate system for identifiability and causal analysis.
We assume that $g:\mathbb{R}^d\rightarrow\mathbb{R}^p$ is smooth and injective
on the support of the latent state distribution, so that the diffusion
representation preserves the information in $s_t$. The latent dynamics follow
a dynamic structural causal model (SCM),
\begin{equation}
s_{i,t+1}
=
F_i\!\left(
s_{\mathrm{Pa}(i),t},
s_{i,t},
u_t,
\varepsilon_{i,t}
\right),
\end{equation}
where $\mathrm{Pa}(i)$ denotes the lagged causal parents of $s_{i,t+1}$ and the
noise variables $\varepsilon_{i,t}$ are mutually independent across coordinates conditional on $u_t$. Throughout, we include the self-lag in the parent set, $i\in\mathrm{Pa}(i)$.

Our objective is to learn $z_t=h(x_t)$ such that the underlying latent
coordinates are identified up to permutation and componentwise invertible
transformations, while preserving the dynamical structure required for graph
recovery and intervention analysis.

\begin{theorem}
[\textbf{Identifiability of nonlinear ICA with auxiliary variables}]
\label{thm:nonlinear-ica}
Assume:
%\begin{enumerate}
    (1) $x_t=g(s_t)$, where $g$ is smooth and injective;
    (2) the latent process is first-order Markov conditional on $u_t$,
    \begin{equation}
    p(s_{t+1}\mid s_{1:t},u_{1:t})
    =
    p(s_{t+1}\mid s_t,u_t);
    \end{equation}
    (3) the transition density factorizes coordinatewise,
    \begin{equation}
    p(s_{t+1}\mid s_t,u_t)
    =
    \prod_{i=1}^d
    p_i(s_{i,t+1}\mid s_t,u_t);
    \end{equation}
    (4) each conditional log-density has exponential-family form
    \begin{equation}
    \log p_i(s_{i,t+1}\mid s_t,u_t)
    =
    \sum_{k=1}^{K_i}
    \lambda_{i,k}(s_t,u_t)
    q_{i,k}(s_{i,t+1})
    -
    \log Z_i(s_t,u_t),
    \end{equation}
    with linearly independent sufficient statistics $q_{i,k}$;
    (5) the modulation coefficients
    $\{\lambda_{i,k}(s_t,u_t)\}_{i,k}$ vary with sufficient rank over the
    auxiliary contexts.
%\end{enumerate}
Then the latent coordinates are identifiable from the diffusion representation
up to permutation and componentwise invertible scalar transformations:
\begin{equation}
z_{i,t}
=
\phi_i(s_{\pi(i),t}),
\qquad i=1,\ldots,d.
\end{equation}
\end{theorem}

\textbf{Contrastive Diffusion Alignment.}
We learn an encoder $h$ on top of the frozen diffusion representation, $z_t=h(x_t)$, using temporal or auxiliary information. ConDA is therefore plug-and-play: the pretrained diffusion model remains fixed and only the alignment map $h$
is learned. For a supervised contrastive formulation, positive pairs are selected using
similarity in auxiliary context:
$\mathcal{P}(i) =
\{p\neq i:d(u_i,u_p)\leq\delta\}$ with negative samples corresponding to auxiliary contexts farther than
$\delta$. The supervised InfoNCE objective is
\begin{equation}
\mathcal{L}_{\mathrm{InfoNCE}}
=
-\frac{1}{N}
\sum_{i=1}^{N}
\log
\frac{
\sum_{p\in\mathcal{P}(i)}
\exp(\mathrm{sim}(z_i,z_p)/\tau)
}{
\sum_{a\in\mathcal{A}(i)}
\exp(\mathrm{sim}(z_i,z_a)/\tau)
},
\label{eq:infonce}
\end{equation}
where $\tau>0$ denotes the temperature parameter. The same framework accommodates different alignment objectives, including TCL-, GCL-, and CEBRA-based variants, which differ in how temporal or auxiliary structure defines positive and negative examples.

%\textbf{Contrastive Density-Ratio Estimation.}
%At the population optimum, supervised contrastive learning can be interpreted as density-ratio estimation.
%The connection between contrastive alignment and nonlinear ICA follows from density-ratio estimation at the population optimum.
\begin{lemma}
[\textbf{Contrastive Density Ratio Estimation}]
\label{lemma-density-ratio}
Let positive pairs be sampled from matched auxiliary contexts and negative
pairs from
\begin{equation}
p_{\mathrm{neg}}(x_t,x_t^-\mid u_t)
=
p(x_t\mid u_t)p(x_t^-).
\end{equation}
If the score class is sufficiently expressive and the population optimum is
attained, then
\begin{equation}
r^*(x_t,x_t^+,u_t)
=
\log
\frac{
p_{\mathrm{pos}}(x_t,x_t^+\mid u_t)
}{
p_{\mathrm{neg}}(x_t,x_t^+\mid u_t)
}
+
C(u_t).
\end{equation}
\end{lemma}
Lemma~\ref{lemma-density-ratio} provides the connection between auxiliary-aligned contrastive learning and auxiliary-variable nonlinear ICA. We next connect the density-ratio representation learned by contrastive alignment to the latent state coordinates.

%\subsection{Identifiability of the Learned Representation}
\begin{theorem}
%[\textbf{Conditional identifiability under separable density-ratio structure}]
[\textbf{Identifiability Through Frozen Diffusion Representations}]
\label{thm:identifiability}
Assume that:
%\begin{enumerate}
    (1) $x_t=g(s_t)$, where $g$ is smooth and injective on the support of $s_t$, with a smooth inverse on its image; 
    (2) the observed auxiliary variable, latent process, sampling scheme, and
    classifier satisfy the assumptions of the applicable TCL or GCL identifiability result \citep{hyvarinen2016unsupervised,hyvarinen2019nonlinear}; and
    (3) any unmixing required by that result is performed.
Then the specified TCL or GCL procedure applied to $x_t$ identifies the latent state up to the equivalence class guaranteed by that result.
If that result gives coordinatewise identification after any required unmixing, its output $z_t$ satisfies
\begin{equation}
z_{i,t}=\phi_i(s_{\pi(i),t}),\qquad i=1,\ldots,d,
\end{equation}
for a permutation $\pi$ and invertible scalar functions $\phi_i$.
\end{theorem}
%\textbf{Interpretation.}
% Theorem~\ref{thm:identifiability} is a sufficient-condition result and does not
% imply that arbitrary InfoNCE training automatically yields coordinate-wise
% identifiable representations. When the separability assumptions are only
% approximately satisfied, the learned representation may instead recover
% mixtures or blocks of latent variables.
\paragraph{Interpretation.}
Theorem~\ref{thm:identifiability} is a sufficient-condition result. It does not imply that arbitrary contrastive training on diffusion representations automatically yields coordinate-wise identifiable variables. When the identification assumptions are only approximately satisfied, the learned representation may recover mixtures or blocks of underlying factors rather than individual coordinates. We therefore assess identifiability empirically using MCC and blockwise $R^2$, together with downstream graph recovery. A generic pairwise similarity objective such as Eq. (7), or the InfoNCE-based CEBRA variant, is not covered by this coordinatewise-identification claim.

%\subsection{Causal Representation Preservation}
\begin{theorem}[\textbf{Preservation of dynamic SCM factorization}]
\label{thm:causal-rep}
Assume the latent dynamics admit the structural factorization
\begin{align}
    s_{i,t+1}
    =
    f_i\!\left(
        s_{\mathrm{Pa}(i),t},
        \varepsilon_{i,t},
        u_t
    \right),
\end{align}
where $\mathrm{Pa}(i)$ denotes the set of lagged parents of $s_{i,t+1}$,
including $i$, and the noises
$\varepsilon_{i,t}$ are mutually independent across $i$ conditional on $u_t$.
Suppose further that the output $z_t$ of the procedure in
Theorem~\ref{thm:identifiability}, after any required unmixing, is coordinatewise identified, so that, after resolving the permutation ambiguity,
\begin{align}
    z_{i,t}=\phi_i(s_{i,t}),
\end{align}
for invertible scalar functions $\phi_i$.
Then the dynamics in the learned coordinates admit the factorization
\begin{align}
    z_{i,t+1}
    =
    \tilde f_i\!\left(
        z_{\mathrm{Pa}(i),t},
        \varepsilon_{i,t},
        u_t
    \right),
\end{align}
where
%\begin{align}
$    \tilde f_i
    =
    \phi_i
    \circ
    f_i
    \circ
    \phi_{\mathrm{Pa}(i)}^{-1}$,
%\end{align}
with the inverse transformations applied componentwise to the parent
coordinates. Consequently, the transformed dynamics preserve the same
lagged parent sets and conditional factorization as the original dynamic SCM.
\end{theorem}

% \subsection{Jacobian-Based Graph Recovery}
% \label{sec:jac}

For the transformed structural model, let $\eta_t = T(z_t,z_{t+1},u_t) $ denote the inverse transition-noise map.
\begin{theorem}
[\textbf{Lagged graph recovery}]
\label{thm:graph-recovery}
Assume that the representation is identifiable up to permutation and
componentwise invertible transformations, each structural mechanism is smooth
and bijective in its noise coordinate, the inverse map $T$ is differentiable
almost everywhere, and the model satisfies structural faithfulness. Then
\begin{equation}
\frac{\partial \eta_{i,t}}
     {\partial z_{j,t}}
=
0
\quad\Longleftrightarrow\quad
j\notin\mathrm{Pa}(i)
\end{equation}
almost everywhere. Hence, the sparsity pattern of the inverse transition-noise
Jacobian recovers the lagged directed causal graph.
\end{theorem}

\paragraph{Forward versus inverse Jacobian.}
Our experiments estimate a forward transition model
$\hat f_\theta(z_t,u_t)
\approx
\mathbb{E}[z_{t+1}\mid z_t,u_t].$
For additive-noise dynamics,
$
z_{i,t+1}
=
\tilde f_i(z_{\mathrm{Pa}(i),t},z_{i,t},u_t)
+
\eta_{i,t},
$
we have
\begin{equation}
\frac{\partial \eta_{i,t}}
     {\partial z_{j,t}}
=
-
\frac{\partial \tilde f_i}
     {\partial z_{j,t}},
\end{equation}
so the forward and inverse Jacobians have the same sparsity pattern. We
therefore use forward-Jacobian magnitudes as empirical edge scores. For
externally imposed controls, derivatives with respect to $u_t$ quantify
control-to-state dependencies separately from the endogenous state graph.
%%%%%%%%%%%%% \textbf{Forward versus inverse Jacobian.}
% Theorem~\ref{thm:graph-recovery} characterizes graph recovery through the
% inverse transition-noise map. In our experiments, we estimate a forward
% transition model in the aligned representation coordinates,
% \[
% \hat f_\theta(\hat{s}_t,u_t)
% \approx
% \mathbb{E}[\hat{s}_{t+1}\mid \hat{s}_t,u_t].
% \]
% For additive-noise dynamics,
% \[
% \hat{s}_{i,t+1}
% =
% \tilde f_i(
% \hat{s}_{\mathrm{Pa}(i),t},
% \hat{s}_{i,t},
% u_t)
% +
% \eta_{i,t},
% \]
% the inverse-noise map satisfies
% \[
% \eta_{i,t}
% =
% \hat{s}_{i,t+1}
% -
% \tilde f_i(
% \hat{s}_{\mathrm{Pa}(i),t},
% \hat{s}_{i,t},
% u_t),
% \]
% and therefore
% \begin{equation}
% \frac{\partial \eta_{i,t}}
%      {\partial \hat{s}_{j,t}}
% =
% -
% \frac{\partial \tilde f_i}
%      {\partial \hat{s}_{j,t}}.
% \end{equation}
% Thus, under additive-noise dynamics, the forward and inverse Jacobians have
% the same sparsity pattern. We therefore use forward-Jacobian magnitudes as
% empirical edge scores. For externally imposed controls, derivatives with
% respect to $u_t$ quantify control-to-state dependencies.

\begin{corollary}
Under the assumptions of
Theorems~\ref{thm:identifiability},
\ref{thm:causal-rep}, and
\ref{thm:graph-recovery},
the learned representation identifies the latent coordinates up to permutation
and componentwise invertible transformations, preserves the dynamic SCM
factorization after permutation resolution, and reduces lagged graph recovery
to recovery of the sparsity pattern of the inverse transition-noise Jacobian.
For additive-noise dynamics, the same sparsity pattern is obtained from the
forward transition Jacobian.
\end{corollary}

\section{Experiments}
%\paragraph{A note on scope.}
%Our pipeline performs post-hoc causal analysis on pretrained diffusion-based representations of dynamical systems. We do not train a new generative world model, nor do we forecast $y_{t+1}$ from $z_t$ in pixel space; the dynamics model $f_\theta$ used for Jacobian-based graph recovery operates within the learned latent space and is not used for generation. Extending ConDA-aligned latents to a generative world model that produces $y_{t+1}$ remains future work.
\textbf{Experimental Setup.}
We evaluate identifiability and lag-1 causal graph recovery on video datasets from
CausalVerse~\citep{causalverse2025}, which contain simulated physical and robotic
systems with known ground-truth causal structure. 
Details of the corresponding reference graphs are provided in Appendices~\ref{app:physical-fall}, \ref{app:robotics-study}, and \ref{app:microduck}. 
We further evaluate
identifiability, controlled dynamics graph recovery, and interventional prediction
on videos generated in Microduck, a bipedal robotic locomotion environment based
on the official Microduck RL framework~\citep{microduck_rl} and simulated in
MuJoCo~\citep{todorov2012mujoco}. The Microduck dataset is generated using walking
and standing policies and, unlike the CausalVerse Physical Fall and Robotics
settings, explicitly includes externally imposed control commands. This enables us
to assess whether learned representations recover not only the underlying physical
state and its dynamical dependencies, but also the system's response to controlled
interventions.

For CausalVerse Physical Fall and Microduck, we use $5000$ videos, with $1000$
held out for testing. For the CausalVerse Robotics Study, we use $2700$ videos,
with $540$ held out for testing. We evaluate ConDA with three contrastive
objectives: CEBRA~\citep{schneider2023cebra},
GCL~\citep{hyvarinen2019nonlinear}, and TCL~\citep{hyvarinen2016unsupervised}.
We compare against five CRL baselines:
IDOL~\citep{li2024identification}, CaRiNG~\citep{chen2024caring},
TDRL~\citep{yao2022temporally}, temporal variant of SuaVE~\citep{liu2026identifying},
and iVAE~\citep{khemakhem2020variational}. All methods are evaluated over five random seeds, and we report the mean and
standard deviation across seeds.

Our pipeline consists of three stages. First, we train a conditional latent
diffusion model (cLDM) on the video data. After training, the diffusion model is
frozen and used only as a pretrained representation model. Each video frame
$y_t$ is first encoded into the latent space of the diffusion model using the
Stable Diffusion VAE~\citep{rombach2022high}. We then apply Rex-RK4 ~\citep{blasingame2026rex} with $8$ inversion steps to the frozen diffusion model to
obtain the corresponding inverted diffusion latent $x_t$.
Second, we learn ConDA representations
$z_t=h(x_t)$ from these frozen diffusion latents. ConDA uses temporal or
auxiliary information $u_t$ through CEBRA, GCL, or TCL to recover structured
latent coordinates. The resulting frame-level embeddings are then assembled
into latent trajectories $\{z_t\}_{t=1}^{T}$ for downstream identifiability,
dynamics, and causal analysis.
Third, we evaluate latent-state recovery and lag-1 graph recovery from the aligned representations. For CausalVerse Physical Fall and Robotics, lagged dependencies are estimated using residual Lasso with trajectory-level bootstrap stability. For Microduck, we instead learn a controlled transition model $[\hat S_t,U_t]\mapsto\hat S_{t+1}$ and estimate state- and control-dependent edges from the average absolute transition Jacobian. Binary graphs are obtained using the reference parent count for each child, applied consistently across methods. This fixes graph density across methods and isolates recovery of the correct parent structure from the separate problem of choosing a sparsity threshold. We therefore interpret binary graph metrics as measuring edge recovery conditional on the true parent count, and additionally report AUPRC as a threshold-free measure of edge-ranking performance. Detailed graph-recovery procedures are given in Algorithms~\ref{alg:residual_lasso} and~\ref{alg:controlled_graph} (Appendix~\ref{app:graph-recovery}).

\textbf{Evaluation Metrics.}
We measure coordinate-wise identifiability by MCC, the mean absolute Pearson correlation between ground-truth state coordinates and learned embedding coordinates under optimal matching, and block-wise identifiability by $R^2$, the variance of ground-truth state blocks explained by nonlinear regression on $z_t$. For graph recovery, we report F1, accuracy, structural Hamming distance (SHD), AUPRC, and structural intervention distance (SID); formal definitions are provided in Appendix~\ref{app:metrics}.
We evaluate interventional prediction in Microduck using $100$ matched simulator
trajectories. Starting from the same state at $t=50$, we compare $do(u_x=\pm 0.25)$ and $do(u_y=\pm 0.17)$ with the zero-command baseline, holding the other commands at zero, and roll out each intervention for 25 steps (1\,s at 25\,Hz). The intervention values lie within the empirical training distribution (Fig.~4). For a command $u_c$ and its target velocity $v$, we define
\begin{equation*}
\mathrm{DIE}_{v}(k) = \mathbb{E}_i\big[v_{i,t+k}^{do(u_c=a)} - v_{i,t+k}^{do(u_c=0)}\big].
\end{equation*}
We compare the MuJoCo response with a true-state dynamics baseline trained on
$[S_t,U_t]$ and a representation-based model trained on
$[\hat S_t,U_t]$, where $\hat S_t$ is obtained from the learned representation; additional details are provided in Appendix \ref{app:DTE}.

\begin{table}[t]
\caption{
\textbf{Identifiability and lag-1 graph recovery across dynamical systems.}
Mean $\pm$ standard deviation across seeds are reported for CausalVerse Physical Fall, CausalVerse Robotics Study, and Microduck. MCC and $R^2$ quantify latent-state recovery, while F1, accuracy, SHD, AUPRC, and SID measure recovery of the lag-1 dynamics graph. ConDA-TCL and ConDA-GCL show consistently strong performance across both identifiability and graph recovery, with exact graph recovery on Physical Fall and competitive performance on the Robotics Study and Microduck. Arrows indicate whether higher ($\uparrow$) or lower ($\downarrow$) values are preferred.
}
\label{tab:causal-comp}
\begin{center}
{
\resizebox{\textwidth}{!}{%
\setlength{\tabcolsep}{2.1pt}   % default is 6pt
\begin{tabular}{lccccccc}
%\begin{tabular}{@{}l*{7}{c}@{}}
\toprule
\textbf{Method} & \textbf{MCC $\uparrow$} & $\mathbf{R^2 \uparrow}$ & \textbf{F1 $\uparrow$} & \textbf{Acc. $\uparrow$} & \textbf{SHD $\downarrow$} & \textbf{AUPRC $\uparrow$} & \textbf{SID $\downarrow$} \\
\midrule
\multicolumn{8}{c}{\textbf{CausalVerse Physical Fall}} \\
\addlinespace[0.5em]
%\multirow{6}{*}{\shortstack[l]{CausalVerse\\Physical Fall}}
CaRiNG        & 0.129 ± 0.013 & 0.497 ± 0.016 & 0.800 ± 0.093 & 0.900 ± 0.046 & 0.400 ± 0.185 & 0.743 ± 0.147 & 0.047 ± 0.021 \\
IDOL          & 0.307 ± 0.014 & 0.418 ± 0.005 & 0.733 ± 0.061 & 0.866 ± 0.030 & 0.533 ± 0.121 & 0.618 ± 0.060 & 0.062 ± 0.014 \\
TDRL          & 0.293 ± 0.045 & 0.480 ± 0.017 & 0.755 ± 0.049 & 0.877 ± 0.024 & 0.488 ± 0.099 & 0.676 ± 0.100 & 0.057 ± 0.011 \\
SuaVE         & 0.438 ± 0.014 & 0.600 ± 0.020 & 0.711 ± 0.060 & 0.855 ± 0.030 & 0.577 ± 0.121 & 0.580 ± 0.072 & 0.068 ± 0.014 \\
iVAE          & \textbf{0.714 ± 0.004} & 0.998 ± 0.000 & 0.977 ± 0.049 & 0.988 ± 0.024 & 0.044 ± 0.099 & 0.983 ± 0.037 & 0.005 ± 0.011 \\
ConDA-CEBRA   & 0.538 ± 0.006 & 0.759 ± 0.001 & 0.800 ± 0.093 & 0.900 ± 0.046 & 0.400 ± 0.185 & 0.696 ± 0.123 & 0.047 ± 0.021 \\
ConDA-TCL     & 0.694 ± 0.001 & 0.998 ± 0.000 & \textbf{1.000 ± 0.000} & \textbf{1.000 ± 0.000} & \textbf{0.000 ± 0.000} & \textbf{1.000 ± 0.000} & \textbf{0.000 ± 0.000} \\
ConDA-GCL     & 0.700 ± 0.002 & \textbf{0.999 ± 0.000} & \textbf{1.000 ± 0.000} & \textbf{1.000 ± 0.000} & \textbf{0.000 ± 0.000} & \textbf{1.000 ± 0.000} & \textbf{0.000 ± 0.000} \\
\midrule
\multicolumn{8}{c}{\textbf{CausalVerse Robotics Study}} \\
\addlinespace[0.5em]
CaRiNG        & 0.509 ± 0.005 & 0.749 ± 0.003 & 0.874 ± 0.006 & 0.988 ± 0.000 & 0.225 ± 0.014 & 0.786 ± 0.013 & 0.010 ± 0.000 \\
IDOL          & 0.412 ± 0.001 & 0.650 ± 0.003 & 0.877 ± 0.000 & 0.988 ± 0.000 & 0.218 ± 0.000 & 0.792 ± 0.000 & 0.010 ± 0.000 \\
TDRL          & 0.481 ± 0.018 & 0.740 ± 0.002 & 0.853 ± 0.012 & 0.986 ± 0.001 & 0.268 ± 0.028 & 0.751 ± 0.024 & 0.010 ± 0.000 \\
SuaVE         & 0.574 ± 0.010 & 0.768 ± 0.002 & 0.858 ± 0.019 & 0.985 ± 0.002 & 0.275 ± 0.040 & 0.771 ± 0.023 & 0.007 ± 0.000 \\
iVAE          & 0.678 ± 0.003 & \textbf{1.000 ± 0.000} & \textbf{0.934 ± 0.016} & \textbf{0.993 ± 0.001} & \textbf{0.125 ± 0.031} & \textbf{0.898 ± 0.016} & \textbf{0.005 ± 0.000} \\
ConDA-CEBRA   & 0.587 ± 0.002 & 0.753 ± 0.003 & 0.889 ± 0.011 & 0.989 ± 0.000 & 0.206 ± 0.017 & 0.815 ± 0.023 & 0.007 ± 0.001 \\
ConDA-TCL     & 0.679 ± 0.005 & \textbf{1.000 ± 0.000} & 0.917 ± 0.011 & 0.992 ± 0.001 & 0.156 ± 0.022 & 0.875 ± 0.013 & 0.006 ± 0.000 \\
ConDA-GCL     & \textbf{0.683 ± 0.002} & \textbf{1.000 ± 0.000} & 0.911 ± 0.026 & 0.991 ± 0.002 & 0.168 ± 0.052 & 0.869 ± 0.016 & 0.006 ± 0.000 \\
\midrule
% \multirow{6}{*}{Microduck}
\multicolumn{8}{c}{\textbf{Microduck}} \\
\addlinespace[0.5em]
CaRiNG        & 0.075 ± 0.025 & 0.323 ± 0.010 & 0.569 ± 0.087 & 0.792 ± 0.042 & 0.207 ± 0.042 & 0.703 ± 0.017 & 0.099 ± 0.031 \\
IDOL          & 0.374 ± 0.003 & 0.357 ± 0.004 & 0.615 ± 0.094 & 0.814 ± 0.045 & 0.185 ± 0.045 & 0.739 ± 0.075 & 0.096 ± 0.033 \\
TDRL          & 0.377 ± 0.010 & 0.383 ± 0.003 & 0.584 ± 0.116 & 0.800 ± 0.056 & 0.200 ± 0.056 & 0.713 ± 0.033 & 0.097 ± 0.045 \\
SuaVE         & 0.378 ± 0.029 & 0.390 ± 0.004 & 0.584 ± 0.042 & 0.800 ± 0.020 & 0.200 ± 0.020 & 0.714 ± 0.053 & 0.124 ± 0.008 \\
iVAE          & 0.626 ± 0.002 & \textbf{1.000 ± 0.000} & 0.753 ± 0.034 & 0.881 ± 0.016 & 0.118 ± 0.016 & 0.798 ± 0.026 & 0.081 ± 0.009 \\
ConDA-CEBRA   & 0.493 ± 0.001 & 0.695 ± 0.006 & 0.507 ± 0.042 & 0.763 ± 0.020 & 0.237 ± 0.020 & 0.682 ± 0.007 & 0.123 ± 0.010 \\
ConDA-TCL     & 0.628 ± 0.001 & \textbf{1.000 ± 0.000} & \textbf{0.769 ± 0.054} & \textbf{0.888 ± 0.026} & \textbf{0.111 ± 0.026} & \textbf{0.796 ± 0.011} & \textbf{0.076 ± 0.017} \\
ConDA-GCL   & \textbf{0.629 ± 0.000} & \textbf{1.000 ± 0.000} & 0.738 ± 0.042 & 0.874 ± 0.020 & 0.125 ± 0.020 & 0.786 ± 0.032 & 0.085 ± 0.011 \\
\bottomrule
\end{tabular}
}}
\end{center}
\end{table}

\section{Results}
\textbf{Identifiability and Graph Recovery.}
Table~\ref{tab:causal-comp} summarizes latent-state identifiability and lag-1 graph
recovery across CausalVerse Physical Fall, CausalVerse Robotics Study, and
Microduck. On Physical Fall, ConDA-TCL and ConDA-GCL achieve near-perfect
block-wise recovery ($R^2\approx 1$) and exact graph recovery
(F1/AUPRC $=1$, SHD/SID $=0$), while iVAE attains the highest MCC
($0.714$). On the Robotics Study, ConDA-TCL and ConDA-GCL achieve the highest
MCC among the ConDA variants ($0.679$ and $0.683$) with $R^2=1.000$; iVAE
performs slightly better on the graph-recovery metrics, whereas ConDA-TCL/GCL
remain competitive. On Microduck, ConDA-GCL attains the highest MCC
($0.629$), while ConDA-TCL gives the strongest graph recovery among the
reported methods in F1 ($0.769$), accuracy ($0.888$), SHD ($0.111$), and
SID ($0.076$), with AUPRC comparable to iVAE. Overall, ConDA-TCL and
ConDA-GCL provide consistently strong performance across both state recovery
and lag-1 dependency recovery, although no single method dominates every
metric and dataset.

\begin{figure}[t]
    \centering
    \includegraphics[width=\textwidth]{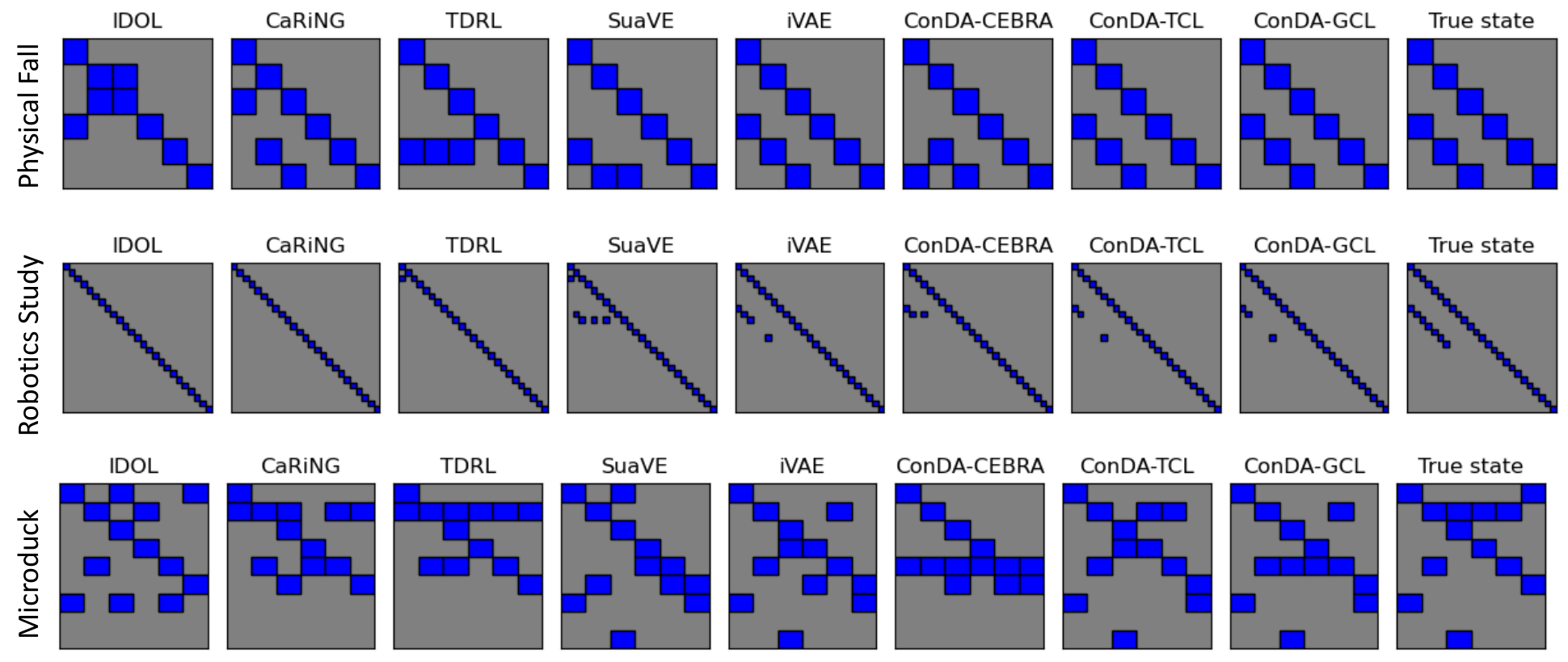}
    \caption{
\textbf{Recovered lag-1 dynamics graphs across datasets and methods.}
Binary adjacency matrices are shown for IDOL, CaRiNG, TDRL, temporal SuaVE, iVAE, ConDA-CEBRA, ConDA-TCL, ConDA-GCL, and the true-state reference. Blue entries indicate recovered edges and gray entries indicate absent edges. Rows correspond to the different dynamical systems considered in the study. The ConDA variants more closely reproduce the sparsity structure of the true-state dynamics, with TCL and GCL showing particularly strong agreement across settings.
}
    \label{fig:hinton}
\end{figure}

Figure~\ref{fig:hinton} visualizes the corresponding binary adjacency matrices.
The exact recovery obtained by ConDA-TCL and ConDA-GCL on Physical Fall is
clearly visible, while the Robotics and Microduck results show that the ConDA
variants preserve much of the sparse reference structure despite several
remaining false positive or missed edges.

\textbf{Block-wise vs.\ Coordinate-wise Recovery.}
Across the three datasets, ConDA-TCL and ConDA-GCL exhibit substantially
stronger block-wise $R^2$ than coordinate-wise MCC. This indicates that the
learned representations retain nearly all information needed to reconstruct the
underlying state, even when individual latent coordinates are not in one-to-one
correspondence with physical state variables. The graph-recovery results further
show that near-perfect block-wise recovery can be sufficient for accurate
recovery of lagged dependency structure, although exact coordinate alignment is
not required in every setting.

\textbf{Dynamic Intervention Effect.}
We further evaluated whether a physical-state readout of the learned representation predicts responses to held-out control perturbations. Figure~\ref{fig:microduck_die}
compares dynamic intervention effects predicted from the TCL-aligned
representation with matched MuJoCo rollouts and a true-state dynamics baseline.
For the backward-velocity intervention, $do(u_x=-0.25)$, TCL recovers
the correct sign and temporal trend, although the response magnitude is
attenuated relative to MuJoCo. For the lateral intervention,
$do(u_y=-0.17)$, TCL follows the simulator response more closely over the
rollout and captures the dominant negative effect on $v_y$. The true-state
baseline generally tracks the transient MuJoCo response more closely, whereas
the representation-based model underestimates the immediate effect. Overall,
the learned representation preserves intervention direction and much of the
longer-horizon dynamics, while exhibiting a smoother and attenuated short-term
response. Additional intervention results are provided in
Appendix~\ref{app:DTE-results}.

\begin{figure}[t]
    \centering
    \includegraphics[width=\textwidth]{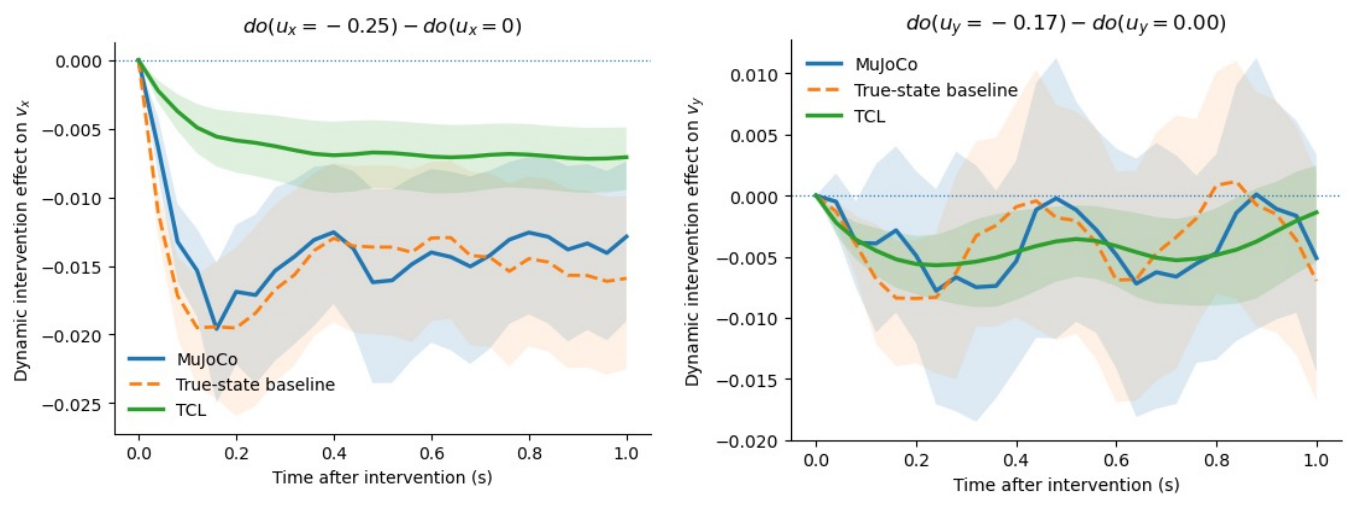}
    \caption{\textbf{Dynamic intervention effects under controlled Microduck perturbations.} Mean dynamic intervention effect over time for interventions on forward velocity command $u_x$ (left) and lateral velocity command $u_y$ (right). The curves compare the MuJoCo simulator response, the true-state dynamics baseline, and the TCL-aligned representation. Shaded regions denote 95\% bootstrap confidence intervals across matched intervention trajectories. The left panel shows the effect of $do(u_x=-0.25)$ relative to $do(u_x=0)$ on $v_x$, while the right panel shows the effect of $do(u_y=-0.17)$ relative to $do(u_y=0)$ on $v_y$.}
    \label{fig:microduck_die}
\end{figure}

\section{Conclusion and Limitations}
We studied when frozen pretrained diffusion representations can support identifiable and causally meaningful world models for high-dimensional dynamical systems. Building on ConDA, we showed that, under stated assumptions, TCL- or GCL-based alignment can recover latent dynamical states up to permutation and componentwise invertible transformations, preserve the factorization and lagged parent structure of the dynamic SCM, and characterize lagged causal dependencies through inverse transition-noise Jacobian sparsity. Empirically, TCL- and GCL-based variants achieved strong blockwise identifiability and supported lagged graph recovery across physical and robotic systems, while Microduck experiments showed preservation of intervention direction and temporal response structure under held-out controls. The theoretical guarantees are population-level and rely on informative auxiliary variables, separability, and regularity assumptions; empirical graph recovery also uses alignment to known state coordinates and, for forward-Jacobian estimation, additive-noise assumptions. Our intervention studies are limited to controlled simulation and open-loop perturbations. Extending the framework to partially observed, higher-order, nonstationary, and real-world scientific systems remains an important direction for future work.

% \bibliography{main}
% \bibliographystyle{plainnat}
\bibliography{main}

@article{blasingame2026rex,
title={Rex: A Family of Reversible Exponential (Stochastic) {Runge-Kutta} Solvers},
author={Zander W. Blasingame and Chen Liu},
journal={Forty-third International Conference on Machine Learning (ICML)},
year={2026},
}

@article{bengio2013representation,
  title={Representation learning: A review and new perspectives},
  author={Bengio, Yoshua and Courville, Aaron and Vincent, Pascal},
  journal={IEEE transactions on pattern analysis and machine intelligence},
  volume={35},
  number={8},
  pages={1798--1828},
  year={2013},
  publisher={IEEE}
}

@inproceedings{causalverse2025,
  title     = {{CausalVerse}: Benchmarking Causal Representation Learning with Configurable High-Fidelity Simulations},
  author    = {Guangyi Chen and Yunlong Deng and Peiyuan Zhu and Yan Li and Yifan Shen and Zijian Li and Kun Zhang},
  booktitle = {NeurIPS},
  year      = {2025},
  url       = {https://huggingface.co/CausalVerse}
}

@article{camps2023discovering,
  title={Discovering causal relations and equations from data},
  author={Camps-Valls, Gustau and Gerhardus, Andreas and Ninad, Urmi and Varando, Gherardo and Martius, Georg and Balaguer-Ballester, Emili and Vinuesa, Ricardo and Diaz, Emiliano and Zanna, Laure and Runge, Jakob},
  journal={Physics Reports},
  volume={1044},
  pages={1--68},
  year={2023},
  publisher={Elsevier}
}

@article{chen2024caring,
  title={{CaRiNG}: Learning temporal causal representation under non-invertible generation process},
  author={Chen, Guangyi and Shen, Yifan and Chen, Zhenhao and Song, Xiangchen and Sun, Yuewen and Yao, Weiran and Liu, Xiao and Zhang, Kun},
  journal={arXiv preprint arXiv:2401.14535},
  year={2024}
}

@inproceedings{chen2020simple,
  title={A simple framework for contrastive learning of visual representations},
  author={Chen, Ting and Kornblith, Simon and Norouzi, Mohammad and Hinton, Geoffrey},
  booktitle={International conference on machine learning (ICML)},
  pages={1597--1607},
  year={2020},
  organization={PmLR}
}

@article{granger1969investigating,
  title={Investigating causal relations by econometric models and cross-spectral methods},
  author={Granger, Clive WJ},
  journal={Econometrica: journal of the Econometric Society},
  pages={424--438},
  year={1969},
  publisher={JSTOR}
}

@article{ha2018world,
  title={World models},
  author={Ha, David and Schmidhuber, J{\"u}rgen},
  journal={arXiv preprint arXiv:1803.10122},
  volume={2},
  number={3},
  pages={440},
  year={2018}
}

@inproceedings{hafner2019learning,
  title={Learning latent dynamics for planning from pixels},
  author={Hafner, Danijar and Lillicrap, Timothy and Fischer, Ian and Villegas, Ruben and Ha, David and Lee, Honglak and Davidson, James},
  booktitle={International conference on machine learning (ICML)},
  pages={2555--2565},
  year={2019},
  organization={PMLR}
}

@article{hafner2020dream,
  title={Dream to control: Learning behaviors by latent imagination},
  author={Hafner, Danijar and Lillicrap, Timothy and Ba, Jimmy and Norouzi, Mohammad},
  journal={International Conference on Learning Representations (ICLR)},
  year={2020}
}

@article{hafner2023mastering,
  title={Mastering diverse domains through world models},
  author={Hafner, Danijar and Pasukonis, Jurgis and Ba, Jimmy and Lillicrap, Timothy},
  journal={arXiv preprint arXiv:2301.04104},
  year={2023}
}

@inproceedings{he2020momentum,
  title={Momentum contrast for unsupervised visual representation learning},
  author={He, Kaiming and Fan, Haoqi and Wu, Yuxin and Xie, Saining and Girshick, Ross},
  booktitle={Proceedings of the IEEE/CVF conference on computer vision and pattern recognition (CVPR)},
  pages={9729--9738},
  year={2020}
}

@article{ho2020denoising,
  title={Denoising diffusion probabilistic models},
  author={Ho, Jonathan and Jain, Ajay and Abbeel, Pieter},
  journal={Advances in neural information processing systems (NeurIPS)},
  volume={33},
  pages={6840--6851},
  year={2020}
}

@article{hyvarinen1999nonlinear,
  title={Nonlinear independent component analysis: Existence and uniqueness results},
  author={Hyv{\"a}rinen, Aapo and Pajunen, Petteri},
  journal={Neural networks},
  volume={12},
  number={3},
  pages={429--439},
  year={1999},
  publisher={Elsevier}
}

@article{hyvarinen2016unsupervised,
  title={Unsupervised feature extraction by time-contrastive learning and nonlinear {ICA}},
  author={Hyvarinen, Aapo and Morioka, Hiroshi},
  journal={NeurIPS},
  year={2016}
}

@inproceedings{hyvarinen2017nonlinear,
  title={Nonlinear {ICA} of temporally dependent stationary sources},
  author={Hyvarinen, Aapo and Morioka, Hiroshi},
  booktitle={Artificial intelligence and statistics},
  pages={460--469},
  year={2017},
  organization={PMLR}
}

@inproceedings{hyvarinen2019nonlinear,
  title={Nonlinear {ICA} using auxiliary variables and generalized contrastive learning},
  author={Hyvarinen, Aapo and Sasaki, Hiroaki and Turner, Richard},
  booktitle={The 22nd international conference on artificial intelligence and statistics},
  pages={859--868},
  year={2019},
  organization={PMLR}
}

@inproceedings{khemakhem2020variational,
  title={Variational autoencoders and nonlinear {ICA}: A unifying framework},
  author={Khemakhem, Ilyes and Kingma, Diederik and Monti, Ricardo and Hyvarinen, Aapo},
  booktitle={International conference on artificial intelligence and statistics},
  pages={2207--2217},
  year={2020},
  organization={PMLR}
}

@inproceedings{locatello2019challenging,
  title={Challenging common assumptions in the unsupervised learning of disentangled representations},
  author={Locatello, Francesco and Bauer, Stefan and Lucic, Mario and Raetsch, Gunnar and Gelly, Sylvain and Sch{\"o}lkopf, Bernhard and Bachem, Olivier},
  booktitle={international conference on machine learning (ICML)},
  pages={4114--4124},
  year={2019},
  organization={PMLR}
}

@article{li2024identification,
  title={On the identification of temporally causal representation with instantaneous dependence},
  author={Li, Zijian and Shen, Yifan and Zheng, Kaitao and Cai, Ruichu and Song, Xiangchen and Gong, Mingming and Chen, Guangyi and Zhang, Kun},
  journal={International Conference on Learning Representations (ICLR)},
  year={2025}
}

@article{liu2026identifying,
  title={Identifying weight-variant latent causal models},
  author={Liu, Yuhang and Zhang, Zhen and Gong, Dong and Gong, Mingming and Huang, Biwei and van den Hengel, Anton and Zhang, Kun and Shi, Javen Qinfeng},
  journal={Journal of Machine Learning Research},
  volume={27},
  number={4},
  pages={1--49},
  year={2026}
}

@misc{microduck_rl,
  author       = {{Pollen Robotics}},
  title        = {Microduck RL: Reinforcement Learning Environments for Microduck},
  year         = {2026},
  howpublished = {\url{https://github.com/pollen-robotics/microduck_rl}},
}

@inproceedings{todorov2012mujoco,
  title={MuJoCo: A physics engine for model-based control},
  author={Todorov, Emanuel and Erez, Tom and Tassa, Yuval},
  booktitle={2012 IEEE/RSJ International Conference on Intelligent Robots and Systems},
  pages={5026--5033},
  year={2012},
  organization={IEEE},
  doi={10.1109/IROS.2012.6386109}
}

@article{oord2018representation,
  title={Representation learning with contrastive predictive coding},
  author={Oord, Aaron van den and Li, Yazhe and Vinyals, Oriol},
  journal={arXiv preprint arXiv:1807.03748},
  year={2018}
}

@book{pearl2009causality,
  title={Causality},
  author={Pearl, Judea},
  year={2009},
  publisher={Cambridge university press}
}

@book{peters2017elements,
  title={Elements of causal inference: foundations and learning algorithms},
  author={Peters, Jonas and Janzing, Dominik and Scholkopf, Bernhard},
  year={2017},
  publisher={MIT press}
}

@inproceedings{rombach2022high,
  title={High-resolution image synthesis with latent diffusion models},
  author={Rombach, Robin and Blattmann, Andreas and Lorenz, Dominik and Esser, Patrick and Ommer, Bj{\"o}rn},
  booktitle={Proceedings of the IEEE/CVF conference on computer vision and pattern recognition (CVPR)},
  pages={10684--10695},
  year={2022}
}

@article{runge2019detecting,
  title={Detecting and quantifying causal associations in large nonlinear time series datasets},
  author={Runge, Jakob and Nowack, Peer and Kretschmer, Marlene and Flaxman, Seth and Sejdinovic, Dino},
  journal={Science advances},
  volume={5},
  number={11},
  pages={eaau4996},
  year={2019},
  publisher={American Association for the Advancement of Science}
}

@article{sandilya2026contrastivediffusionalignmentlearning,
      title={{Contrastive Diffusion Alignment}: Learning Structured Latents for Controllable Generation}, 
      author={Ruchi Sandilya and Sumaira Perez and Charles Lynch and Lindsay Victoria and Benjamin Zebley and Derrick Matthew Buchanan and Mahendra T. Bhati and Nolan Williams and Timothy J. Spellman and Faith M. Gunning and Conor Liston and Logan Grosenick},
      journal={Forty-third International Conference on Machine Learning (ICML)},
      year={2026}, 
}

@article{schneider2023cebra,
  author = {Schneider, Steffen and Lee, Jin H and Mathis, Mackenzie W},
  title = {Learnable latent embeddings for joint behavioural and neural analysis},
  journal = {Nature},
  doi = {https://doi.org/10.1038/s41586-023-06031-6},
  year = {2023},
}

@article{scholkopf2021toward,
  title={Toward causal representation learning},
  author={Sch{\"o}lkopf, Bernhard and Locatello, Francesco and Bauer, Stefan and Ke, Nan Rosemary and Kalchbrenner, Nal and Goyal, Anirudh and Bengio, Yoshua},
  journal={Proceedings of the IEEE},
  volume={109},
  number={5},
  pages={612--634},
  year={2021},
  publisher={IEEE}
}

@inproceedings{sohl2015deep,
  title={Deep unsupervised learning using nonequilibrium thermodynamics},
  author={Sohl-Dickstein, Jascha and Weiss, Eric and Maheswaranathan, Niru and Ganguli, Surya},
  booktitle={International conference on machine learning (ICML)},
  pages={2256--2265},
  year={2015},
  organization={pmlr}
}

@article{song2021score,
  title={Score-based generative modeling through stochastic differential equations},
  author={Song, Yang and Sohl-Dickstein, Jascha and Kingma, Diederik P and Kumar, Abhishek and Ermon, Stefano and Poole, Ben},
  journal={International Conference on Learning Representations (ICLR)},
  year={2021}
}

@book{spirtes2000causation,
  title={Causation, prediction, and search},
  author={Spirtes, Peter and Glymour, Clark N and Scheines, Richard},
  year={2000},
  publisher={MIT press}
}

@article{sun2025causal,
  title={Causal Representation Learning from Multi-modal Biomedical Observations},
  author={Sun, Yuewen and Kong, Lingjing and Chen, Guangyi and Li, Loka and Luo, Gongxu and Li, Zijian and Zhang, Yixuan and Zheng, Yujia and Yang, Mengyue and Stojanov, Petar and others},
  journal={International Conference on Learning Representations (ICLR)},
  pages={arXiv--2411},
  year={2025}
}

@article{tank2021neural,
  title={Neural granger causality},
  author={Tank, Alex and Covert, Ian and Foti, Nicholas and Shojaie, Ali and Fox, Emily B},
  journal={IEEE Transactions on Pattern Analysis and Machine Intelligence},
  volume={44},
  number={8},
  pages={4267--4279},
  year={2021},
  publisher={IEEE}
}

@article{yao2022temporally,
  title={Temporally disentangled representation learning},
  author={Yao, Weiran and Chen, Guangyi and Zhang, Kun},
  journal={NeurIPS},
  volume={35},
  pages={26492--26503},
  year={2022}
}

@article{reizinger2023jacobian,
  title={Jacobian-based causal discovery with nonlinear {ICA}},
  author={Reizinger, Patrik and Sharma, Yash and Bethge, Matthias and Sch{\"o}lkopf, Bernhard and Husz{\'a}r, Ferenc and Brendel, Wieland},
  journal={Transactions on Machine Learning Research (TMLR)},
  year={2023}
}

@article{yao2024marrying,
  title={Marrying causal representation learning with dynamical systems for science},
  author={Yao, Dingling and Muller, Caroline and Locatello, Francesco},
  journal={Advances in Neural Information Processing Systems (NeurIPS)},
  volume={37},
  pages={71705--71736},
  year={2024}
}
\bibliographystyle{iclr2027_conference}

\newpage
\section*{Datasets and Code Access Information}
All code used to generate the results presented in this work will be released as open source under the MIT License. The benchmark dataset is publicly available through Hugging Face at \url{https://huggingface.co/CausalVerse}
\appendix
\section{Appendix}
\subsection{Identifiability of Latent States from Auxiliary-Variable Nonlinear ICA}\label{app:nonlinear-ica-prrof}
\begin{proof}[Proof Sketch of Theorem \ref{thm:nonlinear-ica}]
    Under conditional independence and exponential-family assumptions, the latent variables are identifiable up to permutation and componentwise invertible transformations \cite{hyvarinen2016unsupervised, hyvarinen2019nonlinear, khemakhem2020variational}.
\end{proof}

\subsection{Contrastive Learning as Density Ratio Estimation}
\begin{proof}[Proof of Lemma \ref{lemma-density-ratio}]
At the population level, the supervised contrastive objective \eqref{eq:infonce} can be interpreted
as a classification problem that distinguishes matched pairs drawn from
\(p_{\mathrm{pos}}\) from mismatched pairs drawn from \(p_{\mathrm{neg}}\).
For a sufficiently expressive score class and in the large-negative-sample
limit, the Bayes-optimal score is the log-likelihood ratio between the positive
and negative-pair distributions, up to an additive term depending only on the
conditioning variable \(u_t\).
Therefore,
\[
r^*(x_t,x_t^+,u_t)
=
\log
\frac{
p_{\mathrm{pos}}(x_t,x_t^+\mid u_t)
}{
p_{\mathrm{neg}}(x_t,x_t^+\mid u_t)
}
+
C(u_t).
\]
Substituting
\(p_{\mathrm{neg}}(x_t,x_t^+\mid u_t)=p(x_t\mid u_t)p(x_t^+)\)
gives the stated expression.
\end{proof}

% \paragraph{Remark.}
% This result shows that supervised contrastive learning with matched auxiliary variables reduces, at the population level, to density-ratio estimation between matched and mismatched pairs. This provides the key link between contrastive representation learning and identifiable latent-variable models with auxiliary variables.

\subsection{Identifiability Through Frozen Diffusion Representations}
% \begin{theorem}
% Assume:
% \begin{enumerate}
%     \item The latent process satisfies the assumptions of Theorem 1.
    
%     \item The supervised contrastive objective satisfies the assumptions of Lemma 1.
    
%     \item The population-optimal density-ratio score
%     \begin{align}
%         \log \frac{p_{\mathrm{pos}}(x_t,x_t^+\mid u_t)}{p_{\mathrm{neg}}(x_t,x_t^+\mid u_t)}
%     \end{align}
%     admits a coordinate-separable representation through the true latent coordinates, i.e. it can be written (up to additive terms) as a sum of functions of individual latent components.
    
%     \item The learned score admits a coordinate-separable parametrization through an encoder $h=(h_1,\dots,h_d)$:
%     \begin{align}
%         r_\theta(x_t,x_t^+,u_t)
%         =
%         \sum_{i=1}^d \psi_i\bigl(h_i(x_t),h_i(x_t^+),u_t\bigr),
%     \end{align}
%     and the encoder class is sufficiently rich to realize such representations.
% \end{enumerate}

% Then any global population-optimal representation $z_t = h(x_t)$ identifies the latent state $s_t$ up to permutation and componentwise invertible scalar transformations: there exist a permutation $\pi$ and invertible scalar functions $\phi_i$ such that
% \begin{align}
%     z_{i,t} = \phi_i(s_{\pi(i),t}), \qquad i=1,\dots,d.
% \end{align}

% If the coordinates are ordered by prior knowledge or an auxiliary ordering signal, the permutation ambiguity is removed.
% \end{theorem}

\begin{proof}[Proof Sketch of Theorem \ref{thm:identifiability}]
Because $g$ has a smooth inverse on its image, observing $x_t=g(s_t)$
preserves the state information available to the specified TCL or GCL procedure. Under assumptions (2)--(3), the applicable identification
result yields its stated equivalence class when applied to $x_t$.
If that class is coordinatewise after any required unmixing, its output
satisfies Eq.~(10).

%By Assumption 5, the learned score is constrained to have a compatible coordinate-separable representation through the encoder \(h\), and the model class is sufficiently rich to realize the population-optimal score. Therefore, at a global population optimum, the learned coordinates \(h_i(x_t)\) must realize the same separable latent structure as the true coordinates \(s_{i,t}\), up to the standard ambiguities of auxiliary-variable nonlinear ICA.

%Applying Theorem~\ref{thm:nonlinear-ica}, the recovered coordinates are therefore identified up to permutation and componentwise invertible scalar transformations. Hence, there exist a permutation \(\pi\) and invertible scalar functions \(\phi_i\) such that
%\[
%z_{i,t}
%=
%\phi_i(s_{\pi(i),t}),
%\qquad i=1,\ldots,d.
%\]
\end{proof}

\subsection{Causal Representation Preservation}
% The learned representation preserves the latent SCM:
% \begin{theorem}
% Assume the latent dynamics satisfy
% \begin{align}
%     s_{i,t+1} = f_i(s_{\mathrm{Pa}(i),t}, s_{i,t}, \varepsilon_{i,t}, u_t),
% \end{align}
% where
% \begin{enumerate}
%     \item the graph $\mathcal{G}$ over latent coordinates is a DAG;
%     \item the noises $\varepsilon_{i,t}$ are mutually independent across $i$, conditional on $u_t$;
%     \item each $f_i$ is smooth and injective in $\varepsilon_{i,t}$;
%     \item the learned representation $z_t = h(x_t)$ is identifiable up to permutation and componentwise invertible scalar transformations as in Theorem 2.
% \end{enumerate}
% Then, after resolving the permutation ambiguity, the learned representation is a causal representation in the sense that there exist invertible scalar functions $\phi_i$ such that
% \begin{align}
%     z_{i,t} = \phi_i(s_{i,t}),
% \end{align}
% and the transition model factorizes over the same parent sets:
% \begin{align}
%     z_{i,t+1}
%     =
%     \tilde f_i(z_{\mathrm{Pa}(i),t}, z_{i,t}, \eta_{i,t}, u_t),
% \end{align}
% for suitable transformed mechanisms $\tilde f_i$ and, if needed, componentwise reparameterized noises $\eta_{i,t}$ that remain mutually independent across $i$ conditional on $u_t$.
% \end{theorem}
% Therefore, identifiable coordinates are not merely useful latents; they preserve the latent causal factorization.
\begin{proof}[Proof Sketch of Theorem \ref{thm:causal-rep}]
By Theorem \ref{thm:identifiability}, after resolving the permutation ambiguity there exist invertible scalar functions $\phi_i$ such that
\[
z_{i,t}=\phi_i(s_{i,t}), \qquad i=1,\dots,d.
\]
Substituting $s_{i,t}=\phi_i^{-1}(z_{i,t})$ into the latent SCM
\[
s_{i,t+1}=f_i(s_{\mathrm{Pa}(i),t},s_{i,t},\varepsilon_{i,t},u_t)
\]
yields
\[
z_{i,t+1}
=
\phi_i\!\left(
f_i\bigl(\phi_{\mathrm{Pa}(i)}^{-1}(z_{\mathrm{Pa}(i),t}),
\phi_i^{-1}(z_{i,t}),
\varepsilon_{i,t},u_t
\bigr)
\right).
\]
This defines a transformed mechanism $\tilde f_i$ depending on the same parent coordinates. Because the transformation is componentwise and invertible, it does not mix coordinates and therefore preserves the parent sets. Independence of the noise coordinates is preserved under componentwise reparameterization, yielding the stated transformed SCM.   
\end{proof}

\subsection{Jacobian-Based Graph Recovery}\label{app-jacobian-proof}
% The sparsity pattern of $\frac{\partial \eta_{i,t}}{\partial z_{j,t}}$ recovers the directed causal graph almost everywhere.
% \begin{theorem}
% For the transformed structural model
% \begin{align}
%     z_{i,t+1}
%     =
%     \tilde f_i(z_{\mathrm{Pa}(i),t}, z_{i,t}, \eta_{i,t}, u_t),
% \end{align}
% let $T=(T_1,\dots,T_d)$ denote the inverse transition-noise map defined by
% \begin{align}
%     \eta_t = T(z_t,z_{t+1},u_t).
% \end{align}
% Assume:
% \begin{enumerate}
%     \item the representation $z_t$ is identified up to componentwise invertible transformations and permutation;
%     \item after resolving permutation, the dynamic SCM is acyclic in the coordinate index ordering;
%     \item each mechanism $\tilde f_i$ is smooth and bijective in $\eta_{i,t}$;
%     \item the inverse map $T$ is differentiable almost everywhere on the support;
%     \item (structural faithfulness) for every $j\in \mathrm{Pa}(i)$, the dependence of $T_i(z_t,z_{t+1},u_t)$ on $z_{j,t}$ does not vanish almost everywhere; equivalently,
%     \[
%     \frac{\partial \eta_{i,t}}{\partial z_{j,t}}
%     \neq 0
%     \quad \text{on a set of positive measure.}
%     \]
% \end{enumerate}
% Then, for almost every $(z_t,z_{t+1},u_t)$ in the support,
% \begin{align}
%     \frac{\partial \eta_{i,t}}{\partial z_{j,t}} = 0
%     \quad \Longleftrightarrow \quad
%     j \notin \mathrm{Pa}(i).
% \end{align}
% Hence the sparsity pattern of the Jacobian of the inverse transition-noise map with respect to $z_t$ recovers the directed edges of the latent dynamic causal graph.
% \end{theorem}
\begin{proof}[Proof Sketch of Theorem \ref{thm:graph-recovery}]
Because each mechanism $\tilde f_i$ is bijective in its own noise coordinate $\eta_{i,t}$, the structural equations can be inverted to express each $\eta_{i,t}$ as a function $T_i(z_t,z_{t+1},u_t)$. By construction, $T_i$ depends only on those coordinates of $z_t$ that appear in the corresponding structural equation for $z_{i,t+1}$, namely $z_{i,t}$ and $z_{j,t}$ for $j\in \mathrm{Pa}(i)$. Thus if $j\notin \mathrm{Pa}(i)$, then $\eta_{i,t}$ does not depend on $z_{j,t}$ and
\[
\frac{\partial \eta_{i,t}}{\partial z_{j,t}}=0.
\]
Conversely, if $j\in \mathrm{Pa}(i)$, then by structural faithfulness this dependence does not cancel almost everywhere, so
\[
\frac{\partial \eta_{i,t}}{\partial z_{j,t}}\neq 0
\]
on a set of positive measure. This establishes the equivalence.    
\end{proof}

% \begin{corollary}
% Under the assumptions of Theorems 2, 3, and 4, the learned representation $z_t$ identifies the latent coordinates up to componentwise invertible scalar transformations and permutation, preserves the dynamic SCM factorization after permutation resolution, and therefore reduces recovery of the dynamic causal graph to recovery of the Jacobian sparsity pattern of the inverse transition-noise map. Hence the dynamic causal graph is identifiable almost everywhere.
% \end{corollary}

\subsection{Causal Graph for Physical Fall}\label{app:physical-fall}
We first evaluate our framework on the Physical Fall video dataset from
CausalVerse~\citep{causalverse2025}. This dataset simulates objects undergoing
free-fall motion under gravity. The system evolves in discrete time, and the
underlying state variables describe the object's position and velocity.

For each spatial coordinate, the temporal transition dynamics are given by
\begin{align*}
v_{t+1} &= v_t + g\Delta t, \\
h_{t+1} &= h_t + v_t \Delta t + \tfrac{1}{2}g(\Delta t)^2, 
\end{align*}
where \(h_t\) denotes position or height at time \(t\), \(v_t\) denotes velocity,
\(g\) is gravitational acceleration, and \(\Delta t\) is the simulation time
step. For the three-dimensional setting, the true latent state is
\begin{equation*}
    s_t =
    \left[
    h_{1,t}, h_{2,t}, h_{3,t},
    v_{1,t}, v_{2,t}, v_{3,t}
    \right].
\end{equation*}
The corresponding dynamic causal graph contains three types of directed
time-lagged dependencies:
\begin{align*}
    h_t \rightarrow h_{t+1}, \quad
    v_t \rightarrow v_{t+1}, \quad
    v_t \rightarrow h_{t+1}.
\end{align*}
Thus, each future position depends on its current position and velocity, while
each future velocity depends on its current velocity. We use gravity and time as auxiliary variables. Because the transition mechanism is known analytically, the Physical Fall setting provides a controlled benchmark for evaluating both source identifiability, through alignment between
the learned representations and the true latent states, and dynamic causal graph recovery, through comparison of the estimated lagged adjacency matrix with the known graph.

\subsection{Physics-informed reference lag-1 dynamics graph for Robotics Study}\label{app:robotics-study}
We then evaluate our framework on the Robotics video dataset from CausalVerse~\citep{causalverse2025}, which consists of simulated robotic manipulation scenes in which a robot interacts with objects over time. Compared with the Physical Fall setting, this dataset involves a higher-dimensional state space and more complex temporal dependencies arising from joint motion, end-effector dynamics, and gripper motion.

For each video, we construct a 25-dimensional robot state comprising joint positions and velocities, end-effector position and orientation, and gripper position and velocity:
\begin{equation*}
s_t = \left[ q_t,\, \dot q_t,\, p_t^{\mathrm{eef}},\, Q_t^{\mathrm{eef}},\, g_t,\, \dot g_t \right],
\end{equation*}
where \(q_t,\dot q_t\in\mathbb{R}^{7}\) denote the robot joint positions and velocities, \(p_t^{\mathrm{eef}}\in\mathbb{R}^{3}\) denotes the end-effector position, \(Q_t^{\mathrm{eef}}\in\mathbb{R}^{4}\) its quaternion orientation, and \(g_t,\dot g_t\in\mathbb{R}^{2}\) denote the gripper position and velocity, respectively.

We define a physics-informed reference lag-1 dynamics graph in which every state variable has a self-transition and each joint velocity additionally influences the corresponding joint position at the next time step:
\begin{align*}
q_{k,t}\rightarrow q_{k,t+1},\qquad \dot q_{k,t}\rightarrow q_{k,t+1},\qquad \dot q_{k,t}\rightarrow \dot q_{k,t+1}, \quad k=1,\ldots,7,
\end{align*}
together with
\begin{align*}
p^{\mathrm{eef}}_{j,t}\rightarrow p^{\mathrm{eef}}_{j,t+1},\qquad Q^{\mathrm{eef}}_{j,t}\rightarrow Q^{\mathrm{eef}}_{j,t+1},\qquad g_{j,t}\rightarrow g_{j,t+1},\qquad \dot g_{j,t}\rightarrow \dot g_{j,t+1}.
\end{align*}
This yields 25 self-lag edges and seven cross-lag edges, for a total of 32 reference dynamic edges. We use time, object position and object velocity as auxiliary information for representation learning.

\subsection{Controlled Causal Dynamics for Microduck}\label{app:microduck}
We further evaluate our framework on Microduck, a bipedal robotic locomotion environment implemented using the official Microduck RL
framework~\citep{microduck_rl} and simulated in MuJoCo~\citep{todorov2012mujoco}.
Unlike the Physical Fall and CausalVerse Robotics settings, Microduck explicitly incorporates externally imposed control commands, allowing us to evaluate whether learned representations capture not only the underlying physical state but also its response to interventions.

For each trajectory, we define a six-dimensional physical state
\begin{equation*}
s_t = \left[v_{x,t}, v_{y,t}, \omega_{z,t}, h_t, \phi_t, \theta_t \right],
\end{equation*}
where \(v_{x,t}\) and \(v_{y,t}\) denote the forward and lateral body velocities, \(\omega_{z,t}\) is the body yaw rate, \(h_t\) is the trunk height, and \(\phi_t\) and \(\theta_t\) denote body roll and pitch, respectively. The robot is driven by a three-dimensional external command
\begin{equation*}
u_t = \left[u_{x,t}, u_{y,t}, u_{\omega,t} \right],    
\end{equation*}
corresponding to commanded forward velocity, lateral velocity, and yaw rate. The command \(u_t\) is applied before recording the state at time \(t\) and is held fixed while the simulator advances to \(t+1\). We therefore model the controlled transition mechanism as $ s_{t+1}=F(s_t,u_t)$,
and recover lagged dependencies of the form $ [s_t,u_t]\rightarrow s_{t+1}$. This controlled formulation is also the basis for comparing graph recovery in the true simulator state and in the learned representation.

The control variables have direct physical interpretations. In particular, commanded forward and lateral velocities are expected to influence the corresponding future body velocities, while the commanded yaw rate influences the future body yaw rate:
\begin{align*}
u_{x,t}\rightarrow v_{x,t+1},\qquad u_{y,t}\rightarrow v_{y,t+1},\qquad u_{\omega,t}\rightarrow \omega_{z,t+1}.
\end{align*}
The remaining state variables interact through the robot's locomotion dynamics, including coupling between translational velocity, body attitude, and trunk height. Rather than imposing an arbitrary sparse causal graph, we determine the reference lag-1 structure from the simulator state dynamics and controlled perturbations, using graph recovery from the true state as an oracle benchmark.

\subsection{Evaluation Metric Definitions}
\label{app:metrics}
To evaluate identifiability and causal graph recovery, we use both component-level and block-level metrics.
%To evaluate both component-wise and block-wise identifiability, we use Mean Correlation Coefficient (MCC) and the coefficient of determination $R^2$. Let $s_t \in \mathbb{R}^{d}$ denote the ground-truth state at time $t$ and let $z_t \in \mathbb{R}^{k}$ denote the learned representation. We stack the states over $T$ time points as $ S = [s_1,\ldots,s_T]^\top$ and $Z = [z_1,\ldots,z_T]^\top$. 
Let $s_t \in \mathbb{R}^{d}$ denote the ground-truth state at time $t$, and let $z_t \in \mathbb{R}^{k}$ denote the learned representation. For a trajectory of length $T$, we stack these quantities as $ S = [s_1,\ldots,s_T]^\top$ and $Z = [z_1,\ldots,z_T]^\top$.

\textbf{Mean Correlation Coefficient (MCC).} We use MCC to measure component-wise identifiability between the ground-truth state dimensions and the learned representation dimensions. We first compute the absolute Pearson correlation matrix $ R_{ij} = \left| \operatorname{corr} \left(S_{i}, Z_{j} \right) \right|$, then select an injective matching $\pi:\{1,\ldots,d\}\rightarrow\{1,\ldots,k\}$ maximizing the total matched correlation: $\pi^\star=\arg\max_{\pi}\sum_{i=1}^{d} R_{i,\pi(i)}$.
Finally, the MCC value is defined as $\operatorname{MCC}(S,Z) = \frac{1}{d} \sum_{i=1}^{d} R_{i,\pi^\star(i)}$.

\textbf{Coefficient of Determination ($R^2$).} 
Block-wise $R^2$ measures the proportion of variance in a ground-truth state block
$s_b$ that can be explained by the learned representation $z$. Formally, $R^2=1 -\frac{\operatorname{Var}\left(s_b - \operatorname{Reg}(z)\right)
}{\operatorname{Var}(s_b)}$, where $\operatorname{Reg}$ is a regression function, either linear or nonlinear, trained to predict
the ground-truth state block $s_b$ from the learned representation $z$. A value of $R^2 = 1$ indicates perfect block-wise predictability, meaning that the state block can be fully recovered from the learned representation.

\textbf{Causal Graph Recovery.} 
For causal graph recovery, we compare the estimated adjacency matrix $\hat{B}$ with the ground-truth adjacency matrix $B$, where $B_{ij}=1$ denotes a directed lagged edge from variable $i$ at time $t$ to variable $j$ at time $t+1$. For threshold-based metrics, $\hat{B}$ is binarized and we compute true positives (TP), false positives (FP), false negatives (FN), and true negatives (TN) over all candidate edges. We report
\[\mathrm{Accuracy}
=
\frac{\mathrm{TP}+\mathrm{TN}}
{\mathrm{TP}+\mathrm{FP}+\mathrm{FN}+\mathrm{TN}},
\qquad
\mathrm{Precision}
=
\frac{\mathrm{TP}}
{\mathrm{TP}+\mathrm{FP}},
\]
\[
\mathrm{Recall}
=
\frac{\mathrm{TP}}
{\mathrm{TP}+\mathrm{FN}},
\qquad
\mathrm{F1}
=
\frac{2\,\mathrm{Precision}\,\mathrm{Recall}}
{\mathrm{Precision}+\mathrm{Recall}}.
\]
To evaluate graph recovery independently of a single threshold, we additionally report the area under the precision--recall curve (AUPRC), computed from the continuous edge scores. We report structural Hamming distance (SHD), which quantifies structural disagreement between $\hat{B}$ and $B$ through missing, extra, and incorrectly oriented edges; lower values indicate better recovery, with $\mathrm{SHD}=0$ corresponding to exact structural recovery. We also report structural intervention distance (SID), which measures disagreement between the estimated and ground-truth graphs in the interventional distributions implied by their graph structures; lower SID indicates greater agreement in causal effects, with $\mathrm{SID}=0$ indicating no structural intervention errors. We report SHD and SID normalized by the number of candidate edges. Figure~\ref{fig:hinton} shows the binarized adjacency matrices, with recovered edges in blue.

\subsection{Dynamic Intervention Effect}\label{app:DTE}
To assess whether the learned representation supports prediction under
controlled interventions, we perform matched interventional rollouts in
the Microduck environment. For each of 100 trajectories, the simulator is
reconstructed to the same state at intervention time $t=50$. We then compare $do(u_x=\pm 0.25)$ and $do(u_y=\pm 0.17)$ with the corresponding zero-command baseline, holding the remaining commands at zero. The intervention values lie within the
empirical support of the training control distribution
(Fig.~\ref{fig:microduck_control_distribution}). 

\begin{figure}[t]
    \centering
    \includegraphics[width=\textwidth]{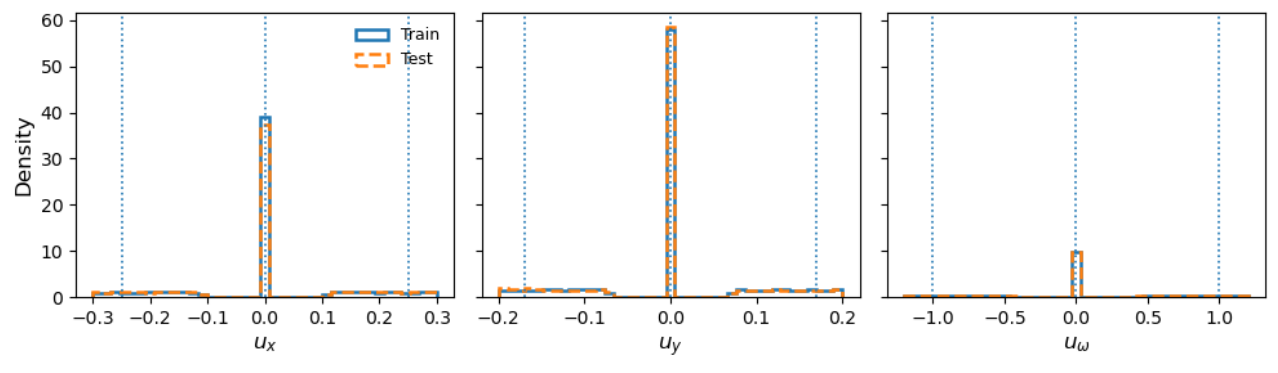}
    \caption{
    \textbf{Distribution of Microduck control commands.}
    Empirical distributions of the forward-velocity command $u_x$,
    lateral-velocity command $u_y$, and yaw-rate command $u_\omega$
    in the training and test sets. The distributions closely overlap
    across the two splits. Dotted vertical lines indicate the negative,
    zero, and positive command values used in the intervention experiments.
    }
    \label{fig:microduck_control_distribution}
\end{figure}

Each intervention is
propagated for 25 simulation steps, corresponding to $1\,\mathrm{s}$ at
25~Hz.
For a command $u_c\in\{u_x,u_y\}$ and state component $s_j$, we define the dynamic intervention effect (DIE) at horizon $k$ as
\begin{equation}
\mathrm{DIE}_j(k;a,b) = \mathbb{E}_i \left[s^{do(u_c=a)}_{i,t+k,j} - s^{do(u_c=b)}_{i,t+k,j} \right],
\end{equation}
where the expectation is over matched trajectories $i$. We evaluate the direct effects of $u_x$ on $v_x$ and of $u_y$ on $v_y$.

We compare the MuJoCo simulator response with two learned dynamics models:
(i) a true-state baseline trained directly on the physical state, $[S_t,U_t] \longrightarrow S_{t+1}$,
and (ii) a representation-based model in which the learned latent representation is first aligned to the physical state, $Z_t \longrightarrow \hat S_t$,
after which the controlled dynamics are modeled as $[\hat S_t,U_t] \longrightarrow \hat S_{t+1}$.

For a rollout horizon $H$ and target velocity $v\in\{v_x,v_y\}$, we quantify intervention prediction error as
\begin{align}
\mathrm{RMSE}_{H} = \sqrt{
\frac{1}{H}
\sum_{k=1}^{H}
\left(
\widehat{\mathrm{DIE}}_{v}(k)
-
\mathrm{DIE}^{\mathrm{sim}}_{v}(k)
\right)^2
},
\end{align}
and report the corresponding relative error
\begin{align}
\mathrm{RelErr}_{H}
=
\frac{\mathrm{RMSE}_{H}}
{
\sqrt{
\frac{1}{H}
\sum_{k=1}^{H}
\left(
\mathrm{DIE}^{\mathrm{sim}}_{v}(k)
\right)^2
}
}.
\end{align}

\begin{figure}[t]
    \centering
    \includegraphics[width=\linewidth]{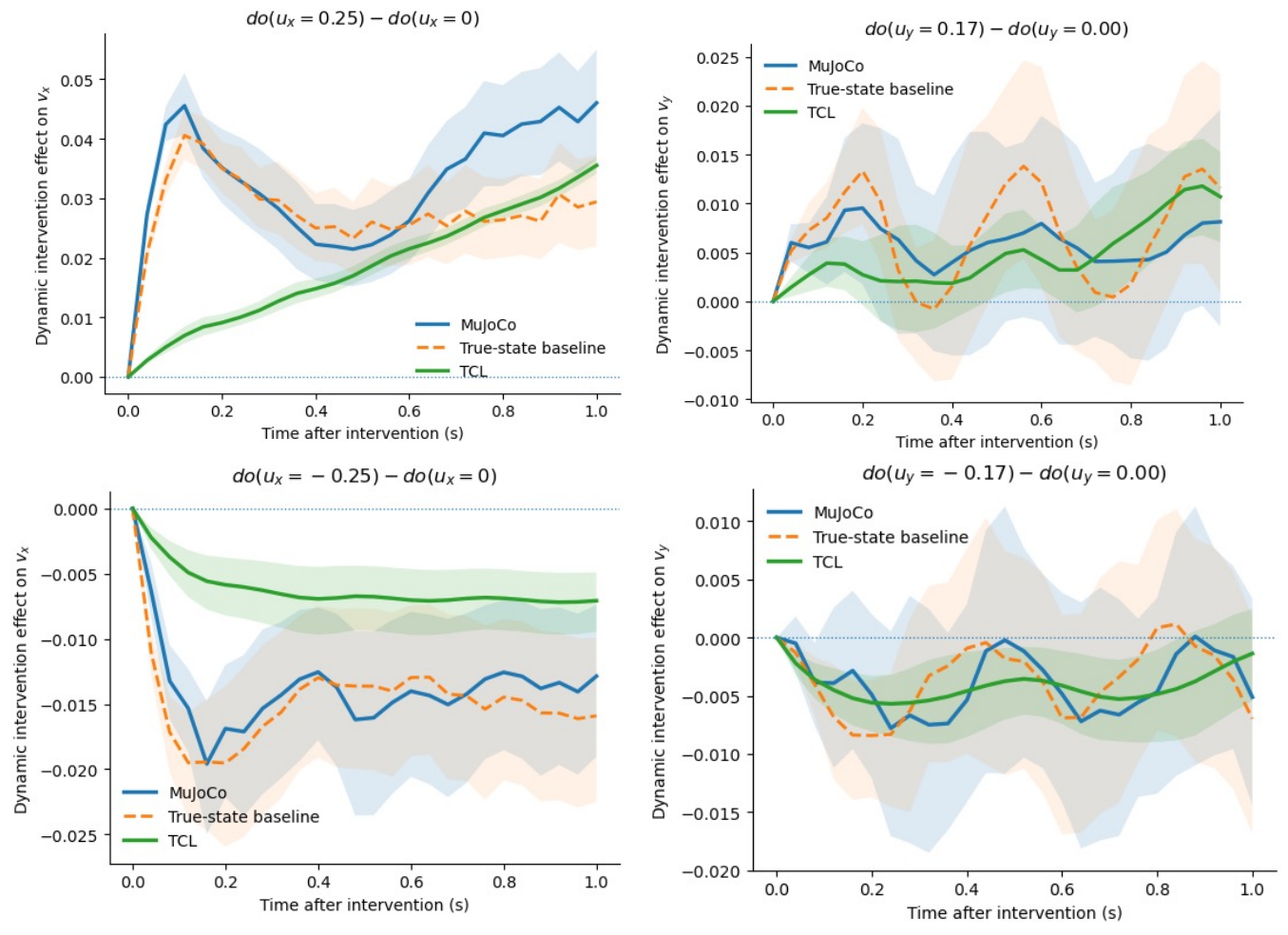}
    \caption{\textbf{Dynamic intervention effects under positive and negative control perturbations in Microduck.} Mean dynamic intervention effects over a 1s rollout horizon for interventions on the forward command \(u_x\) and lateral command \(u_y\). The top row shows positive interventions, $do(u_x=0.25)$ and $do(u_y=0.17)$, relative to the corresponding zero-control baselines, while the bottom row shows negative interventions, \(do(u_x=-0.25)\) and \(do(u_y=-0.17)\). Responses are measured on \(v_x\) for \(u_x\) interventions and on \(v_y\) for \(u_y\) interventions. Curves compare the MuJoCo simulator, the true-state dynamics baseline, and the TCL-aligned representation. Shaded regions indicate 95\% bootstrap confidence intervals across matched intervention trajectories.}
    \label{fig:microduck-DIE-all}
\end{figure}

\subsubsection{Results} \label{app:DTE-results}

\begin{table}[t]
% ============================================================
% Row 1: Positive interventions
% ============================================================

\begin{minipage}[H]{0.48\textwidth}
    \centering
    \captionof{table}{
    \textbf{Dynamic intervention effect on forward velocity in Microduck.}
    Errors are computed relative to the matched MuJoCo response for
    $do(u_x=0.25)$ versus $do(u_x=0)$.
    }\vspace{0.6em}
    \label{tab:microduck_die_pos_ux}

    \small
    \setlength{\tabcolsep}{3.2pt}
    \renewcommand{\arraystretch}{1.08}

    \begin{tabular}{ccccccc}
    \toprule
    \textbf{$H$} &
    \textbf{$t$ (s)} &
    \multicolumn{2}{c}{\textbf{RMSE $\downarrow$}} &
    \textbf{Rel. $\downarrow$} &
    \multicolumn{2}{c}{\textbf{Endpoint}} \\
    \cmidrule(lr){3-4}
    \cmidrule(lr){6-7}
    & &
    \textbf{Rep.} &
    \textbf{State} &
    \textbf{Rep.} &
    \textbf{Sim.} &
    \textbf{Rep.} \\
    \midrule
    1  & 0.04 & 0.0248 & 0.0065 & 0.910 & 0.0273 & 0.0025 \\
    2  & 0.08 & 0.0320 & 0.0092 & 0.897 & 0.0424 & 0.0046 \\
    5  & 0.20 & 0.0323 & 0.0072 & 0.845 & 0.0351 & 0.0085 \\
    10 & 0.40 & 0.0259 & 0.0059 & 0.772 & 0.0223 & 0.0130 \\
    25 & 1.00 & 0.0199 & 0.0110 & 0.571 & 0.0461 & 0.0281 \\
    \bottomrule
    \end{tabular}
\end{minipage}
\hfill
\begin{minipage}[H]{0.48\textwidth}
    \centering
    \captionof{table}{
    \textbf{Dynamic intervention effect on lateral velocity in Microduck.}
    Errors are computed relative to the matched MuJoCo response for
    $do(u_y=0.17)$ versus $do(u_y=0)$.
    }\vspace{0.6em}
    \label{tab:microduck_die_pos_uy}

    \small
    \setlength{\tabcolsep}{3.2pt}
    \renewcommand{\arraystretch}{1.08}

    \begin{tabular}{ccccccc}
    \toprule
    \textbf{$H$} &
    \textbf{$t$ (s)} &
    \multicolumn{2}{c}{\textbf{RMSE $\downarrow$}} &
    \textbf{Rel. $\downarrow$} &
    \multicolumn{2}{c}{\textbf{Endpoint}} \\
    \cmidrule(lr){3-4}
    \cmidrule(lr){6-7}
    & &
    \textbf{Rep.} &
    \textbf{State} &
    \textbf{Rep.} &
    \textbf{Sim.} &
    \textbf{Rep.} \\
    \midrule
    1  & 0.04 & 0.0046 & 0.0009 & 0.765 & 0.0060 & 0.0014 \\
    2  & 0.08 & 0.0038 & 0.0014 & 0.660 & 0.0055 & 0.0027 \\
    5  & 0.20 & 0.0047 & 0.0023 & 0.626 & 0.0097 & 0.0028 \\
    10 & 0.40 & 0.0041 & 0.0028 & 0.630 & 0.0038 & 0.0019 \\
    25 & 1.00 & 0.0036 & 0.0035 & 0.572 & 0.0084 & 0.0107 \\
    \bottomrule
    \end{tabular}
\end{minipage}

\vspace{1.2em}

% ============================================================
% Row 2: Negative interventions
% ============================================================

\begin{minipage}[H]{0.48\textwidth}
    \centering
    \captionof{table}{
    \textbf{Dynamic intervention effect on forward velocity in Microduck.}
    Errors are computed relative to the matched MuJoCo response for
    $do(u_x=-0.25)$ versus $do(u_x=0)$.
    }\vspace{0.6em}
    \label{tab:microduck_die_neg_ux}

    \small
    \setlength{\tabcolsep}{3.2pt}
    \renewcommand{\arraystretch}{1.08}

    \begin{tabular}{ccccccc}
    \toprule
    \textbf{$H$} &
    \textbf{$t$ (s)} &
    \multicolumn{2}{c}{\textbf{RMSE $\downarrow$}} &
    \textbf{Rel. $\downarrow$} &
    \multicolumn{2}{c}{\textbf{Endpoint}} \\
    \cmidrule(lr){3-4}
    \cmidrule(lr){6-7}
    & &
    \textbf{Rep.} &
    \textbf{State} &
    \textbf{Rep.} &
    \textbf{Sim.} &
    \textbf{Rep.} \\
    \midrule
    1  & 0.04 & 0.0042 & 0.0042 & 0.672 & -0.0063 & -0.0021 \\
    2  & 0.08 & 0.0075 & 0.0038 & 0.728 & -0.0132 & -0.0034 \\
    5  & 0.20 & 0.0109 & 0.0030 & 0.727 & -0.0169 & -0.0050 \\
    10 & 0.40 & 0.0104 & 0.0021 & 0.700 & -0.0125 & -0.0047 \\
    25 & 1.00 & 0.0102 & 0.0021 & 0.711 & -0.0129 & -0.0042 \\
    \bottomrule
    \end{tabular}
\end{minipage}
\hfill
\begin{minipage}[H]{0.48\textwidth}
    \centering
    \captionof{table}{
    \textbf{Dynamic intervention effect on lateral velocity in Microduck.}
    Errors are computed relative to the matched MuJoCo response for
    $do(u_y=-0.17)$ versus $do(u_y=0)$.
    }\vspace{0.6em}
    \label{tab:microduck_die_neg_uy}

    \small
    \setlength{\tabcolsep}{3.2pt}
    \renewcommand{\arraystretch}{1.08}

    \begin{tabular}{ccccccc}
    \toprule
    \textbf{$H$} &
    \textbf{$t$ (s)} &
    \multicolumn{2}{c}{\textbf{RMSE $\downarrow$}} &
    \textbf{Rel. $\downarrow$} &
    \multicolumn{2}{c}{\textbf{Endpoint}} \\
    \cmidrule(lr){3-4}
    \cmidrule(lr){6-7}
    & &
    \textbf{Rep.} &
    \textbf{State} &
    \textbf{Rep.} &
    \textbf{Sim.} &
    \textbf{Rep.} \\
    \midrule
    1  & 0.04 & 0.0017 & 0.0008 & 3.705 & -0.0005 & -0.0022 \\
    2  & 0.08 & 0.0012 & 0.0006 & 0.452 & -0.0039 & -0.0036 \\
    5  & 0.20 & 0.0014 & 0.0033 & 0.387 & -0.0048 & -0.0056 \\
    10 & 0.40 & 0.0016 & 0.0034 & 0.286 & -0.0056 & -0.0045 \\
    25 & 1.00 & 0.0019 & 0.0028 & 0.403 & -0.0049 & -0.0014 \\
    \bottomrule
    \end{tabular}
\end{minipage}

\end{table}

To assess whether the learned representation preserves intervention-relevant
dynamics, we evaluated controlled rollouts under matched positive and negative
perturbations of the forward and lateral velocity commands. For each
intervention, we compared the dynamic intervention effect predicted from the
aligned TCL representation with the corresponding MuJoCo response and with a
true-state dynamics baseline. The resulting intervention trajectories are shown
in Fig.~\ref{fig:microduck-DIE-all}, and the corresponding
horizon-wise quantitative results are summarized in
Tables~\ref{tab:microduck_die_pos_ux}--\ref{tab:microduck_die_neg_uy}.

For interventions on the forward command $u_x$, the learned representation
recovered the correct direction of the response for both positive and negative
perturbations, but generally underestimated its magnitude
(Fig.~\ref{fig:microduck-DIE-all};
Tables~\ref{tab:microduck_die_pos_ux} and
\ref{tab:microduck_die_neg_ux}). Under $do(u_x=0.25)$, the representation RMSE
decreased from $0.0248$ at $H=1$ to $0.0199$ at $H=25$, while the relative
error decreased from $0.910$ to $0.571$. The endpoint response at $1$~s was
$0.0281$ for the representation compared with $0.0461$ in MuJoCo. For the
negative intervention $do(u_x=-0.25)$, the representation similarly preserved
the sign of the effect but produced an attenuated response. At $H=25$, the
representation endpoint was $-0.0042$, compared with $-0.0129$ in the
simulator, with a relative error of $0.711$. The true-state dynamics baseline
was consistently more accurate for the forward-velocity interventions,
indicating that part of the remaining error arises from representation recovery
in addition to finite-horizon dynamics prediction.

The lateral interventions were recovered more accurately
(Fig.~\ref{fig:microduck-DIE-all};
Tables~\ref{tab:microduck_die_pos_uy} and
\ref{tab:microduck_die_neg_uy}). For $do(u_y=0.17)$, the representation RMSE
remained below $0.0048$ across all evaluated horizons and reached $0.0036$ at
$H=25$, with a relative error of $0.572$. The endpoint effect at $1$~s was
$0.0107$, compared with $0.0084$ in MuJoCo. For $do(u_y=-0.17)$, the
representation achieved low absolute error throughout the rollout, with RMSE
values between $0.0012$ and $0.0019$ for $H \geq 2$. At $H=10$, the relative
error was $0.286$, with a representation endpoint of $-0.0045$ compared with a
simulator effect of $-0.0056$. The larger relative error at $H=1$ for this case
is due to the very small simulator effect magnitude at the first prediction
step. For this intervention, the representation-based model also attains lower RMSE than the true-state baseline for $H\geq 5$ (Table~\ref{tab:microduck_die_neg_uy}), consistent with the small effect magnitude relative to rollout variability.

Overall, Fig.~\ref{fig:microduck-DIE-all} and
Tables~\ref{tab:microduck_die_pos_ux}--\ref{tab:microduck_die_neg_uy} show that
the aligned TCL representation preserves the qualitative direction and temporal
structure of controlled responses across both intervention signs and both
control dimensions. Recovery is strongest for lateral-velocity perturbations,
while forward-velocity effects are systematically attenuated in magnitude. The
comparison with the true-state baseline further separates
representation-related error from error introduced by the learned rollout
dynamics.

\subsection{Graph-Recovery Algorithms} \label{app:graph-recovery}
\begin{algorithm}[t]
\caption{Residual-Lasso graph recovery with trajectory-bootstrap stability}
\label{alg:residual_lasso}
\begin{algorithmic}[1]
\Require Aligned state trajectories
$\{\hat S^{(n)}\}_{n=1}^{N}$,
$\hat S^{(n)}\in\mathbb{R}^{T_n\times d}$;
reference parent counts $\{k_j\}_{j=1}^{d}$;
Lasso penalty $\lambda$;
number of bootstrap samples $B$;
trajectory sampling fraction $\rho$.
\Ensure Stability scores $\Pi\in[0,1]^{d\times d}$ and binary lag-1 graph
$\hat B\in\{0,1\}^{d\times d}$.

\State Standardize the aligned states using statistics estimated from the
training trajectories.

\For{$b=1,\ldots,B$}
    \State Resample a fraction $\rho$ of complete trajectories.
    \State Construct lag-one predictors and residual targets
    \[
    X_t=\hat S_t,
    \qquad
    R_t=\hat S_{t+1}-\hat S_t .
    \]

    \For{$j=1,\ldots,d$}
        \State Exclude $\hat S_{t,j}$ from the candidate cross-parent set and fit
        \[
        \hat\beta_j^{(b)}
        =
        \arg\min_{\beta}
        \frac{1}{2M_b}
        \|R_{:,j}-X_{:,-j}\beta\|_2^2
        +
        \lambda\|\beta\|_1 .
        \]

        \State Let $r_j=k_j-1$ denote the number of cross-lag parents
        for child $j$.
        \State Record the $r_j$ predictors with the largest
        $|\hat\beta_{ij}^{(b)}|$.
    \EndFor
\EndFor

\State Compute the trajectory-bootstrap stability score
\[
\Pi_{ij}
=
\frac{1}{B}
\sum_{b=1}^{B}
\mathbbm{1}
\!\left[
i \text{ is selected as a parent of } j
\right].
\]

\State Initialize $\hat B=0$ and include the reference self-lag edge
$\hat B_{jj}=1$ for each state variable.

\For{$j=1,\ldots,d$}
    \State Select the $r_j=k_j-1$ cross-parents with the largest
    stability scores $\Pi_{ij}$ and set the corresponding
    $\hat B_{ij}=1$.
\EndFor

\State \Return $\Pi,\hat B$.
\end{algorithmic}
\end{algorithm}

We use trajectory-level rather than frame-level resampling so that temporal
dependence within each trajectory is preserved. The reference parent count
$k_j$ is used only for graph binarization and is applied identically to all
methods, thereby controlling graph density and separating dependency ranking
from threshold calibration. AUPRC is additionally reported from the continuous
edge scores without fixing a binary threshold.

\begin{algorithm}[t]
\caption{Controlled lag-1 graph recovery in Microduck}
\label{alg:controlled_graph}
\begin{algorithmic}[1]
\Require Aligned state trajectories
$\{\hat S^{(n)}\}_{n=1}^{N}$,
$\hat S_t\in\mathbb{R}^{d}$;
control trajectories
$\{U^{(n)}\}_{n=1}^{N}$,
$U_t\in\mathbb{R}^{q}$;
reference parent counts $\{k_j\}_{j=1}^{d}$.
\Ensure Dependency scores
$J\in\mathbb{R}^{(d+q)\times d}$ and binary controlled graph
$\hat B\in\{0,1\}^{(d+q)\times d}$.

\State Construct controlled lag-one pairs
\[
X_t=[\hat S_t,U_t],
\qquad
Y_t=\hat S_{t+1}.
\]

\State Standardize $X_t$ and $Y_t$ using training-set statistics.

\State Train a neural transition model
$f_\theta:\mathbb{R}^{d+q}\rightarrow\mathbb{R}^{d}$ by minimizing
\[
\mathcal{L}(\theta)
=
\frac{1}{M}
\sum_{t}
\left\|
Y_t-f_\theta(X_t)
\right\|_2^2 .
\]

\State Compute the average absolute transition Jacobian
\[
J_{ij}
=
\frac{1}{M}
\sum_t
\left|
\frac{\partial f_{\theta,j}(X_t)}
     {\partial X_{t,i}}
\right|,
\qquad
i=1,\ldots,d+q,
\quad
j=1,\ldots,d .
\]

\State Initialize $\hat B=0$.

\For{$j=1,\ldots,d$}
    \State Select the $k_j$ inputs with the largest values of $J_{ij}$:
    \[
    \mathcal{P}_j
    =
    \operatorname{TopK}
    \left(
    \{J_{ij}\}_{i=1}^{d+q},k_j
    \right).
    \]
    \For{$i\in\mathcal{P}_j$}
        \State Set $\hat B_{ij}=1$.
    \EndFor
\EndFor

\State \Return $J,\hat B$.
\end{algorithmic}
\end{algorithm}
For Microduck, $S_t=[v_x,v_y,\omega_z,h,\phi,\theta],
\quad U_t=[u_x,u_y,u_\omega],$
so the transition model has nine candidate state/control inputs and six
next-state outputs. The same controlled-dynamics procedure is applied to the
true physical state to obtain the true-state dynamics baseline.

% \begin{algorithm}[H]
% \caption{Matched dynamic intervention effect in Microduck} \label{alg:microduck_die}
% \begin{algorithmic}[1]
% \Require Matched trajectories $i=1,\ldots,N_I$;
% intervention time $t^\star$;
% intervention values $a,b$;
% rollout horizon $H$;
% simulator and learned controlled-dynamics model.
% \Ensure Simulator and model dynamic intervention effects.

% \For{$i=1,\ldots,N_I$}
%     \State Reconstruct the simulator to the same state
%     $S_{i,t^\star}$ by replaying the original trajectory up to $t^\star$.

%     \State From this matched state, generate two simulator rollouts under
%     $do(u_x=a)$ and $do(u_x=b)$ for $H$ steps.

%     \State From the corresponding model state, generate matched model rollouts
%     under the same two interventions.
% \EndFor

% \For{$k=1,\ldots,H$}
%     \State Compute the simulator dynamic intervention effect
%     \[
%     \mathrm{DIE}^{\mathrm{sim}}_j(k;a,b)
%     =
%     \frac{1}{N_I}
%     \sum_{i=1}^{N_I}
%     \left(
%     s^{do(a)}_{i,t^\star+k,j}
%     -
%     s^{do(b)}_{i,t^\star+k,j}
%     \right).
%     \]

%     \State Compute $\widehat{\mathrm{DIE}}_j(k;a,b)$ analogously from
%     the learned model rollouts.
% \EndFor

% \State Evaluate horizon-wise RMSE and relative error between
% $\widehat{\mathrm{DIE}}$ and $\mathrm{DIE}^{\mathrm{sim}}$.

% \State Estimate confidence intervals by paired trajectory bootstrap.

% \State \Return
% $\mathrm{DIE}^{\mathrm{sim}},
% \widehat{\mathrm{DIE}}$ and intervention-error metrics.
% \end{algorithmic}
% \end{algorithm}

% #############################
\subsection{Detailed Related Work}
\textbf{World models and latent dynamical representations.}
World models learn latent representations that support prediction, planning, and
control in high-dimensional environments
\citep{ha2018world,hafner2019learning,hafner2020dream,hafner2023mastering}.
These models have achieved strong performance in reinforcement learning and
sequential decision-making, where the primary objectives are typically
reconstruction, forecasting, reward prediction, or control. However, predictive
accuracy does not ensure that the learned latent state corresponds to the
underlying physical or causal state of the system. This distinction is especially
important in scientific settings, where intervention and mechanistic
interpretation require representations that are not only predictive but also
identifiable and structurally meaningful. Our work studies whether frozen
diffusion representations can be transformed into latent states that recover the
underlying dynamical variables and their directed dependencies.

\textbf{Diffusion models and latent representations.}
Diffusion models are powerful generative models for high-dimensional data,
including images and videos
\citep{sohl2015deep,ho2020denoising,song2021score}. Latent diffusion models
improve scalability by performing generation in compressed representation spaces
rather than directly in pixel space \citep{rombach2022high}. Although these
representations often contain rich semantic and structural information, they are
primarily optimized for generative fidelity rather than identifiability or causal
interpretability. Their coordinates therefore need not correspond to the latent
dynamical sources governing the observations. We build on frozen diffusion
representations and ask whether auxiliary-aligned low-dimensional embeddings can
recover identifiable dynamical states and support structural causal analysis.

\textbf{Contrastive representation learning.}
Contrastive learning learns structured representations by bringing positive pairs
together while separating negative pairs
\citep{oord2018representation,chen2020simple,he2020momentum}. The information
captured by the learned embedding depends strongly on how positive and negative
pairs are constructed. In dynamical settings, time-contrastive learning (TCL)
uses temporal nonstationarity to recover latent sources under nonlinear ICA
assumptions
\citep{hyvarinen2016unsupervised,hyvarinen2017nonlinear}, while generalized
contrastive learning (GCL) uses auxiliary variables to obtain identifiability
under broader forms of side information \citep{hyvarinen2019nonlinear}. CEBRA
similarly learns low-dimensional embeddings aligned with behavioral, temporal,
or neural variables \citep{schneider2023cebra}. In our setting, these objectives
serve as auxiliary-alignment mechanisms applied to frozen diffusion latents, and
we evaluate whether the resulting embeddings recover identifiable states and
support downstream causal structure recovery.

\textbf{Nonlinear ICA and identifiability with auxiliary variables.}
Nonlinear latent-variable models are generally non-identifiable without
additional assumptions
\citep{hyvarinen1999nonlinear,locatello2019challenging}. Nonlinear ICA provides
conditions under which latent sources can nevertheless be recovered from
nonlinear observations. In particular, auxiliary information such as time,
labels, experimental conditions, or context can provide sufficient structure for
identifiability
\citep{hyvarinen2016unsupervised,hyvarinen2017nonlinear,hyvarinen2019nonlinear,
khemakhem2020variational}. These results motivate our theoretical analysis. We
connect auxiliary-aligned contrastive learning on diffusion latents to
density-ratio estimation and auxiliary-variable nonlinear ICA, and use this
connection to establish sufficient conditions for identifying latent dynamical
states from frozen pretrained representations.

\textbf{Causal representation learning.}
Causal representation learning seeks latent representations that support causal
reasoning, intervention, and transfer
\citep{bengio2013representation,scholkopf2021toward}. A central challenge is
that representations may capture predictive or statistical regularities without
preserving the variables and mechanisms required for causal reasoning. Recent
work therefore studies conditions under which learned coordinates correspond to
causal variables or preserve structural mechanisms. Our analysis contributes to
this direction by showing that componentwise identifiable dynamical coordinates
preserve the factorization and lagged parent structure of the underlying dynamic
SCM. This links latent identifiability directly to preservation of causal
structure in time-dependent systems.

\textbf{Causal discovery in dynamical systems.}
Causal discovery aims to recover directed dependencies from observational or
interventional data
\citep{spirtes2000causation,pearl2009causality,peters2017elements}. In
time-series and dynamical systems, temporal ordering and lagged dependencies
provide additional structure for identifying directed interactions
\citep{granger1969investigating,runge2019detecting,tank2021neural}. Many such
methods, however, assume that the observed variables already correspond to
meaningful system states. In high-dimensional observations such as video, these
variables may instead be latent, making causal discovery dependent on the quality
of the learned representation. We address this issue by first establishing
conditions for identifiable state recovery and then characterizing lagged causal
edges through sparsity of the inverse transition-noise Jacobian. This connection
is related to Jacobian-based causal discovery under nonlinear ICA
\citep{reizinger2023jacobian}, but is developed here for auxiliary-aligned
representations of dynamical systems and time-lagged causal structure.

\textbf{Causal representation learning for scientific dynamical systems.}
Recent work has increasingly connected causal representation learning with
scientific dynamical modeling. ~\citet{yao2024marrying} study the integration of identifiable causal representations with dynamical-system models
for scientific applications, while~\citet{sun2025causal} develop identification theory for causal representation learning from multimodal biomedical observations. More broadly,~\citet{camps2023discovering} review methods for discovering causal relations,
equations, and scientific laws from observational data. Our work is
complementary, focusing specifically on frozen pretrained diffusion
representations of high-dimensional dynamical systems and evaluating whether
auxiliary alignment recovers latent state structure, lagged causal dependencies,
and intervention-relevant dynamics.

\end{document}